\documentclass{article}

\usepackage{amsmath, amssymb, amsxtra, graphicx, float, caption, color, bm, fullpage, amsthm}
\usepackage[position=b]{subcaption}
\usepackage{natbib}

\usepackage{enumitem}

\newcommand{\ep}{\mathbb{E}}

\newcommand{\E}{{\cal E}}
\newcommand{\wijbkwn}{{W(\phi_{ij}, \beta_k; w_n)}}

\newtheorem{theorem}{Theorem}

\numberwithin{equation}{section}
\theoremstyle{plain} 

\newtheorem{lemma}[theorem]{Lemma}

\newcommand{\prb}{\mathbb{P}}

\begin{document}

\title{Community Detection in Networks with Node Features}
\author{Yuan Zhang, Elizaveta Levina, Ji Zhu}
\maketitle

\begin{abstract}
	Many methods have been proposed for community detection in networks, but most of them do not take into account additional information on the nodes that is often available in practice. In this paper, we propose a new joint community detection criterion that uses both the network edge information and the node features to detect community structures.  One advantage our method has over existing joint detection approaches is the flexibility of learning the impact of different features which may differ across communities.  Another advantage is the flexibility of choosing the amount of influence the feature information has on communities.  The method is asymptotically consistent under the block model with additional assumptions on the feature distributions, and performs well on simulated and real networks.   
\end{abstract}

Community detection is a fundamental problem in network analysis, extensively studied in a number of domains -- see  \cite{korean1} and \cite{schlitt2007current} for some examples of applications.  A number of approaches to community detection are based on probabilistic models for networks with communities, such as the stochastic block model \cite{Holland83}, the degree-corrected stochastic block model \cite{Karrer10}, and the latent factor model \cite{Hoff2007}.  Other approaches work by optimizing a criterion measuring the strength of community structure in some sense, often through spectral approximations.  Examples include normalized cuts \cite{shi00},  modularity \cite{Newman&Girvan2004, NewmanPNAS}, and many variants of spectral clustering, e.g.,  \cite{qin2013regularized}.

Many of the existing methods detect communities based only on the network adjacency matrix.  However, we often have additional information on the nodes (node features), and sometimes edges as well, for example, \cite{steglich2006applying}, \cite{snijders2006new} and \cite{hummon1990analyzing}.  In many networks the distribution of node features is correlated with community structure \cite{Jure_ego}, and thus a natural question is whether we can improve community detection by using the node features.   Several generative models for jointly modeling the edges and the features have been proposed, including the network random effects model \cite{hoff2003random}, the embedding feature model \cite{Zanghi:2010:CBR:1774858.1774960}, the latent variable model \cite{RSSA:RSSA471}, the discriminative approach \cite{Yang:2009:CLC:1557019.1557120}, the latent multi-group membership graph model \cite{Jure_multiple_attr}, the social circles model for ego networks \cite{Jure_ego}, the communities from edge structure and node attributes (CESNA) model \cite{yang2013community}, the Bayesian Graph Clustering (BAGC) model \cite{xu2012model} and the topical communities and personal interest (TCPI) model \cite{hoang2014joint}.   Most of these models are designed for specific feature types, and their effectiveness depends heavily on the correctness of model specification.  Model-free approaches include weighted combinations of the network and feature similarities \cite{viennet2012community, binkiewicz2014covariate}, attribute-structure mining \cite{Silva:2012:MAC:2140436.2140443}, simulated annealing clustering \cite{Cheng:2011:CLA:1921632.1921638}, and compressive information flow \cite{smith2014partitioning}.  Most methods in this category use all the features in the same way without determining which ones influence the community structure and which do not, and lack flexibility in how to balance the network information with the information coming from its node features, which do not always agree.   Including irrelevant node features can only hurt community detection by adding in noise, while selecting features that by themselves cluster strongly may not correspond to features that correlate with the community structure present in the adjacency matrix.

In this paper, we propose a new joint community detection criterion that uses both the network adjacency matrix and the node features. The idea is that by properly weighing edges according to feature similarities on their end nodes, we strengthen the community structure in the network thus making it easier to detect.   Rather than using all available features in the same way,  we learn which features are most helpful in identifying the community structure from data.   Intuitively,  our method looks for an agreement between clusters suggested by two data sources, the adjacency matrix and the node features.  Numerical experiments on simulated and real networks show that our method performs well compared to methods that use either the network alone or the features alone for clustering, as well as to a number of benchmark joint detection methods.

\section{The joint community detection criterion}
Our method is designed to look for assortative community structure, that is, the type of communities where nodes are more likely to connect to each other if they belong to the same community, and thus there are more edges within communities than between.  This is a very common intuitive definition of communities which is incorporated in many community detection  criteria, for example, modularity \cite{NewmanPNAS}.   Our goal is to use such a community detection criterion based on the adjacency matrix alone, and add feature-based edge weights to improve detection.   Several criteria using the adjacency matrix alone are available, but having a simple criterion linear in the adjacency matrix makes optimization much more feasible in our particular situation, and we propose a new criterion which turns out to work particularly well for our purposes.   Let $A$ denote the adjacency matrix with $A_{ij}=0$ if there is no edge between nodes $i$ and $j$,  and otherwise $A_{ij} > 0$  which can be either 1 for unweighted networks or the edge weight for weighted networks.    The community detection criterion we start from is a very simple analogue of modularity, to be maximized over all possible label assignments $e$: 
\begin{equation}
  R(e)=\sum_{k=1}^K \frac{1}{|\E_k|^\alpha}\sum_{i, j\in \E_k}A_{ij} \ . 
     \label{marginal_criterion}
\end{equation}
Here $e$ is the vector of node labels, with $e_i = k$ if node $i$ belongs to community $k$, for $k = 1, \dots, K$,  $\E_k = \{i: e_i = k\}$, and $|\E_k|$ is the number of nodes in community $k$.  We assume each node belongs to exactly one community, and the number of communities $K$ is fixed and known.  Rescaling by $|\E_k|^\alpha$ is designed to rule out trivial solutions that put all nodes in the same community, and $\alpha > 0$ is a tuning parameter.  When $\alpha=2$, the criterion is approximately the sum of edge densities within communities, and when $\alpha=1$, the criterion is the sum of average ``within community'' degrees, which both intuitively represent community structure.   This criterion can be shown to be consistent under the stochastic block model by checking the conditions of the general theorem in \cite{Bickel&Chen2009}.  

The ideal use of features with this criterion would be to use them to up-weigh edges within communities and down-weigh edges between them, thus enhancing the community structure in the observed network and making it easier to detect.  However, node features may not be perfectly correlated with community structure, different communities may be driven by different features, as pointed out by \cite{Jure_ego}, and features themselves may be noisy.  Thus we need to learn the impact of different features on communities as well as balance the roles of the network itself and its features.   Let $f_i$ denote the $p$-dimensional feature vector of node $i$.   We propose a \emph{joint community detection criterion} (JCDC), 
\begin{equation}
R(e, \beta; w_n) =\sum_{k=1}^K\frac{1}{|\E_k|^\alpha}\sum_{i,j\in \E_k}A_{ij}W(f_i, f_j, \beta_k; w_n) \label{JCDC}
\end{equation}
where $\alpha$ is a tuning parameter as in \eqref{marginal_criterion}, $\beta_k \in \mathbb{R}^p$ is the coefficient vector that defines the impact of different features on the $k$th community, and $\beta:=\{\beta_1,\ldots, \beta_K\}$.   The criterion is then maximized over both $e$ and $\beta$.   Having a different $\beta_k$ for each $k$ allows us to learn the roles different features may play in different communities.  The balance between the information from $A$ and $F:=\{f_1,\ldots,f_n\}$ is controlled by $w_n$, another tuning parameter which in general may depend on $n$. 

For the sake of simplicity, we model the edge weight $W(f_i, f_j, \beta_k; w_n)$ as a function of the node features $f_i$ and $f_j$ via a $p$-dimensional vector of their similarity measures $\phi_{ij} = \phi(f_i, f_j)$. The choice of similarity measures in $\phi$ depends on the type of $f_i$ (for example, on whether the features are numerical or categorical) and is determined on a case by case basis;  the only important property is that $\phi$ assigns higher values to features that are more similar.  Note that this trivially allows the inclusion of edge features  as well as node features, as long as they are converted to some sort of similarity.     To eliminate potential differences in units and scales, we standardize all $\phi_{ij}$ along each feature dimension.   Finally, the function $W$ should be increasing in $\langle \phi_{ij}, \beta \rangle$, which can be viewed as the ``overall similarity'' between nodes, and for optimization purposes it is convenient to take $W$ to be concave.  Here we use the exponential function, 
\begin{equation}
w_{ijk} = W(f_i, f_j, \beta_k; w_n) = 
w_n-e^{-\langle \phi_{ij},\beta_k \rangle}
\label{weight_function}
\end{equation}
One can use other functions of similar shapes, for example, the logit exponential function, which we found empirically to perform similarly.

\section{Estimation}
\label{sec:optimization}
The joint community detection criterion needs to be optimized over both the community assignments $e$ and the feature parameters  $\beta$.   Using block coordinate descent, we optimize JCDC by alternately optimizing over the labels with fixed parameters and over the parameters with fixed labels, and iterating until convergence. 

\subsection{Optimizing over label assignments with fixed weights}
\label{sec:fixedbeta}
When parameters $\beta$ are fixed,  all edge weights $w_{ijk}$'s can be treated as known constants.  It is infeasible to search over all $n^K$ possible label assignments, and, like many other community detection methods, we rely on a greedy label switching algorithm to optimize over $e$, specifically, the tabu search \cite{Glover:tabu}, which updates the label of one node at a time.   Since our criterion involves the number of nodes in each community $|\E_k|$, no easy spectral approximations are available.   Fortunately, our method allows for a simple local approximate update which does not require recalculating the entire criterion.  For a given node $i$ considered for label switching, the algorithm will assign it to community $k$ rather than $l$ if  
	\begin{align}
		\frac{S_{kk}+2S_{i\leftrightarrow k}}{(|\E_k|+1)^\alpha} + \frac{S_{ll}}{|\E_l|^\alpha}  & >  \frac{S_{kk}}{|\E_k|^\alpha} + \frac{S_{ll}+2S_{i\leftrightarrow l}}{(|\E_l|+1)^\alpha} \ ,  \label{k_vs_l_raw}
	\end{align}
where $S_{kk}$ is twice the total edge weights in community $k$, and $S_{i\leftrightarrow k}$ is the sum of edge weights between node $i$ and all the nodes in $\E_k$.  When $|\E_k|$ and $|\E_l|$ are large, we can ignore $+1$ in the denominators, and \eqref{k_vs_l_raw} becomes 
\begin{equation}
		\frac{S_{i \leftrightarrow k}}{|\E_k|}\cdot \frac{|\E_k|^{1-\alpha}}{|\E_l|^{1-\alpha}} > \frac{S_{i \leftrightarrow l}}{|\E_l|} \ , 
\label{alpha_and_community_size}
\end{equation}
which allows for a ``local'' update for the label of node $i$ without calculating the entire criterion.  This also highlights the impact of the tuning parameter $\alpha$:  when $\alpha=1$, the two sides of \eqref{alpha_and_community_size} can be viewed as averaged weights of all edges connecting node $i$ to communities $\E_k$ and $\E_l$, respectively.  Then our method assigns node $i$ to the community with which it has the strongest connection.    When $\alpha \neq 1$, the left hand side of \eqref{alpha_and_community_size} is multiplied by a factor $(|\E_k|/|\E_l|)^{1-\alpha}$.  Suppose $|\E_k|$ is larger than $|\E_l|$;  then choosing $0 < \alpha < 1$ indicates a preference for assigning a node to the larger community, while  $\alpha>1$ favors smaller communities.  A detailed numerical investigation of the role of $\alpha$ is provided in the Supplemental Material.   

 The edge weights involved in \eqref{alpha_and_community_size} depend on the tuning parameter $w_n$. When $\beta=0$, all weights are equal to $w_n-1$.  On the other hand, $w_{ijk} \le w_n$ for all values of $\beta$. Therefore, $w_n/(w_n-1)$ is the maximum amount by which our method can  reweigh an edge. When $w_n$ is large, $w_n/(w_n-1)\approx 1$, and thus the information from the network structure dominates. When $w_n$ is close to $1$, the ratio is large and the feature-driven edge weights have a large  impact.    See the Supplemental Material for more details on the choice of $w_n$.  

While the tuning parameter $w_n$ controls the amount of influence features can have on community detection, it does not affect the estimated parameters $\beta$ for a fixed community assignment.  This is easy to see from rearranging terms in \eqref{JCDC}: 
\begin{equation}
R(e, \beta; w_n) = w_n \sum_{k=1}^K\frac{1}{|\E_k|^\alpha}\sum_{i, j\in \E_k}A_{ij} - g(e, A, \beta, \phi) 
\label{break_JCDC_to_two_terms}
\end{equation}
where the function $g$ does not depend on $w_n$.  Note that the term containing $w_n$ does not depend on $\beta$.

\subsection{Optimizing over weights with fixed label assignments}

Since we chose a concave edge weight function \eqref{weight_function}, for a given community assignment $e$ the joint criterion is a concave function of $\beta_k$, and it is straightforward to optimize over  $\beta_k$ by gradient ascent.    The role of $\beta_k$ is to control the impact of different features on each community.     One can show by a Taylor-series type expansion around the maximum (details omitted) and also observe empirically that for our method, the estimated $\hat{\beta}_k$'s are correlated with the feature similarities between nodes in community $k$.  
In other words, our method tends to produce a large estimated $\hat{\beta}_k^{(\ell)}$ for a feature with high similarity values $\phi_{ij}^{(\ell)}$'s for $i, j \in \E_k$.  However, in the extreme case, the optimal $\hat{\beta}_{k}^{(\ell)}$ can be $+ \infty$ if all $\phi_{ij}^{(\ell)}$'s are positive in community $k$ or $- \infty$ if all $\phi_{ij}^{(\ell)}$'s are negative (recall that similarities are standardized, so this cannot happen in all communities).   To avoid these extreme solutions, we subtract a penalty term $\lambda\|\beta\|_1$ from the criterion \eqref{JCDC} while optimizing over $\beta$.    We use a very small value of $\lambda$ ($\lambda =10^{-5}$ everywhere in the paper) which safeguards against numerically unstable solutions but has very little effect on other estimated coefficients.  


\section{Consistency}
The proposed JCDC criterion \eqref{JCDC} is not model-based, but under certain models it is asymptotically consistent. We consider the setting where the network $A$ and the features $F$ are generated independently from a stochastic block model and a uniformly bounded distribution, respectively.   Let $\mathbb{P}(A_{ij}=1) = \rho_n P_{c_i c_j}$ where $\rho_n$ is a factor controling the overall edge density and $c = (c_1, \dots, c_n)$ is the vector of true labels.   Assume the following regularity conditions hold:
\begin{enumerate}
\item There exist global constants $M_\phi$ and $M_\beta$, such that $\|\phi_{ij}\|_2\leq M_\phi$ and $\|\beta_k\|_2\leq M_\beta$ for all $k$,  
and the tuning parameter $w_n$ satisfies  $\log w_n> M_\phi M_\beta$. \label{condition_1}
\item Let ${\cal C}_k:=\{i: c_i = k\}$. There exists a global constant $\pi_0$ such that $|C_k|\geq \pi_0 n > 0$ for all $k$.  
\label{condition_2}
\item  For all $1\leq k<l \leq K$,  $2(K-1)P_{kl}<\min(P_{kk}, P_{ll})$. \label{condition_3}
\end{enumerate}

Condition \ref{condition_1} states that node feature similarities are uniformly bounded.  This is a mild condition in many applications as the node features are often themselves uniformly bounded. In practice, for numerical stability the user may want to standardize node features and discard individual features with very low variance, before calculating the corresponding similarities $\phi$. Condition \ref{condition_2} guarantees communities do not vanish asymptotically.   Condition \ref{condition_3}  enforces assortativity.  
Since the estimated labels $e$ are only defined up to an arbitrary permutation of communities, we measure the agreement betwee $e$ and $c$ by 
$d(e, c) =\min_{\sigma \in {\cal P}_K}  \frac{1}{n} \sum_{i = 1}^n {\mathbf 1}(\sigma(e_i) \neq c_i )$, where ${\cal P}_K$ is the set of all permutations of $\{1,\ldots,K\}$.

\begin{theorem}[Consistency of JCDC] Under conditions \ref{condition_1}, \ref{condition_2}  and \ref{condition_3}, 
if $n\rho_n\to\infty$, $w_n \rho_n\to \infty$, and the parameter $\alpha$ satisfies 
\begin{equation}
\frac{\max_{k, l}2(K-1)P_{kl}}{\min_{k,l}(P_{kk}, P_{ll})}\leq\alpha\leq 1 
\end{equation}
then we have, for any fixed $\delta>0$, 
\begin{equation}
\mathbb{P}\left( d\left(\arg\max_{e}(\max_{\beta} R(e,\beta;w_n)), c\right) > \delta \right) \to 0 \ . \label{main_theorem_equation}
\end{equation}
\end{theorem}
The proof  is given in the Supplemental Material.

\section{Simulation studies}

We compare JCDC to three representative benchmark methods which use both the adjacency matrix and the node features: CASC (Covariate Assisted Spectral Clustering, \cite{binkiewicz2014covariate}), CESNA (Communities from Edge Structure and Node Attributes, \cite{yang2013community}), and BAGC (BAyesian Graph Clustering, \cite{xu2012model}).   In addition, we also include two standard methods that use either the network adjacency alone (SC, spectral clustering on the Laplacian regularized with a small constant $\tau = 1e-7$, as in \cite{Amini.et.al2013}), or the node features alone (KM, $K$-means performed on the $p$-dimensional node feature vectors, with 10 random initial starting values).   We generate networks with $n=150$ nodes and $K=2$ communities of sizes $100$ and $50$ from the degree-corrected stochastic block model as follows. The edges are generated independently with probability $\theta_i\theta_j p$ if nodes $i$ and $j$ are in the same community, and $r \theta_i\theta_j p$ if nodes $i$ and $j$ are in different communities.    We set $p=0.1$ and vary $r$ from $0.25$ to $0.75$.    We set $5\%$ of the nodes in each community to be ``hub'' nodes with the degree correction parameter $\theta_i=10$, and for the remaining nodes set $\theta_i=1$.   All resulting products are thresholded at 0.99 to ensure there are no probability values over 1.  These settings result in the average expected node degree ranging approximately from $22$ to $29$.

\begin{figure}[H]
	\centering
	\begin{subfigure}[h]{0.135\textwidth}
		\centering
		\includegraphics[width=\textwidth]{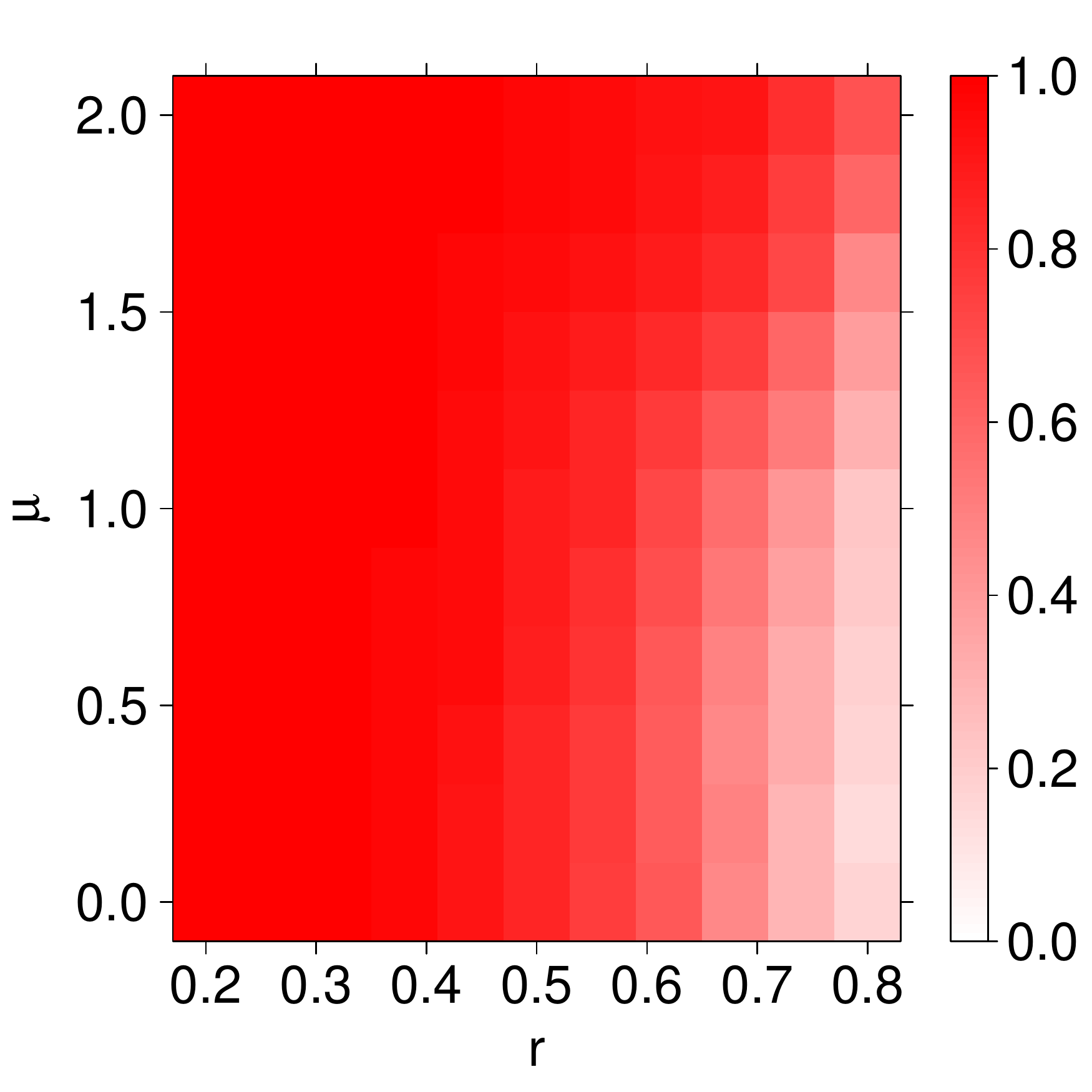}
		\caption{JCDC,$w=5$}
	\end{subfigure}
	\begin{subfigure}[h]{0.135\textwidth}
		\centering
		\includegraphics[width=\textwidth]{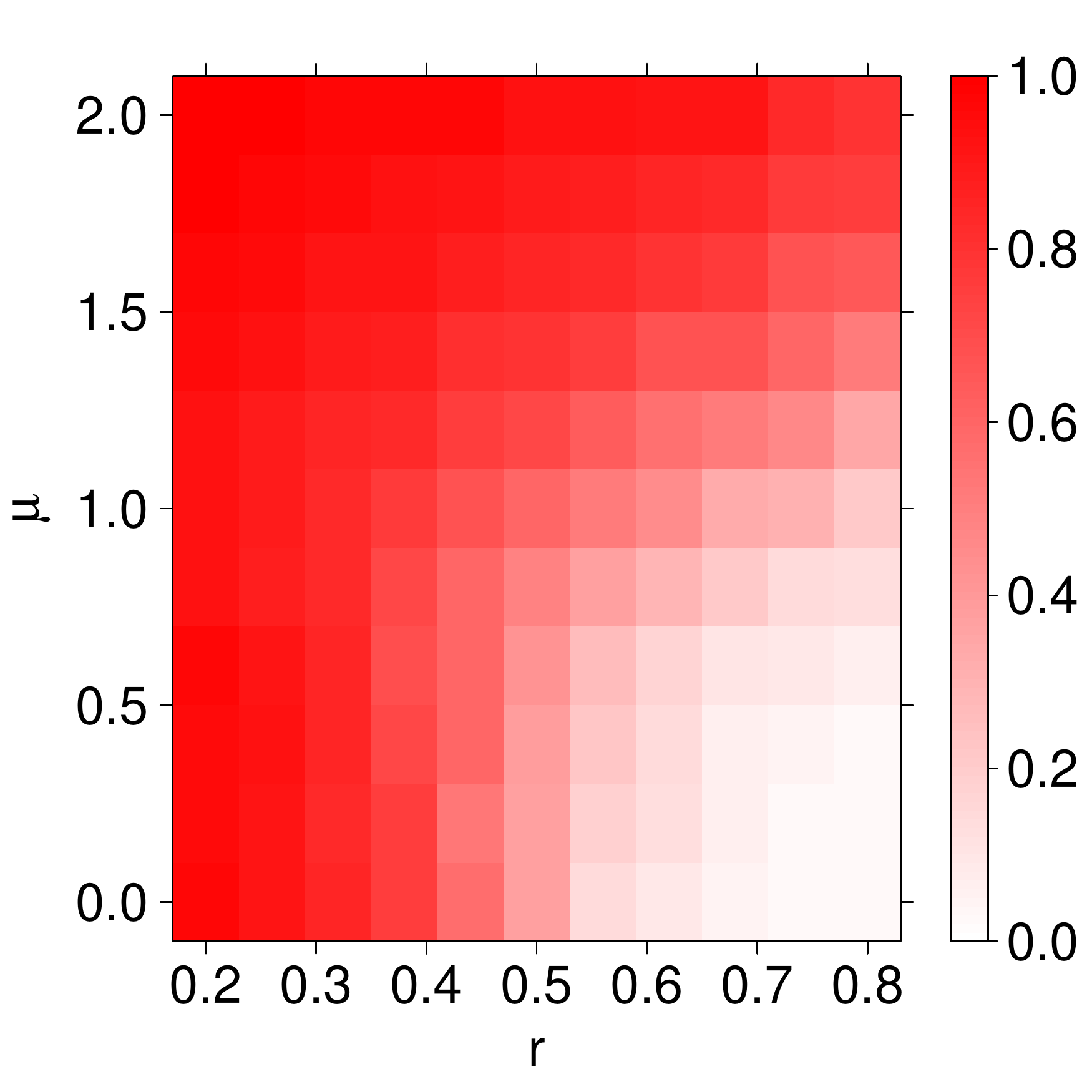}
		\caption{JCDC,$w=1.5$}
	\end{subfigure}
	\begin{subfigure}[h]{0.135\textwidth}
		\centering
		\includegraphics[width=\textwidth]{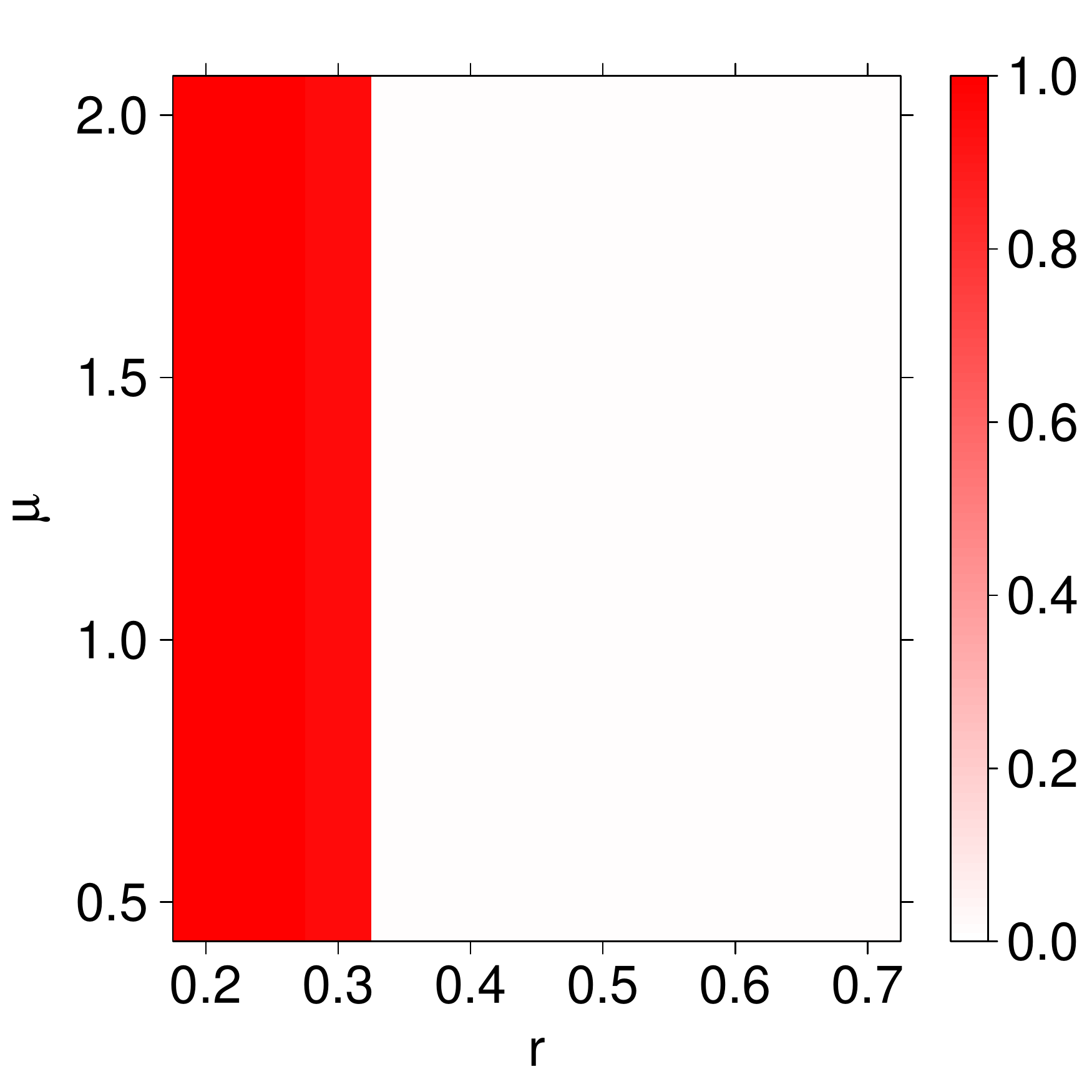}
		\caption{SC}
	\end{subfigure}
	\begin{subfigure}[h]{0.135\textwidth}
		\centering
		\includegraphics[width=\textwidth]{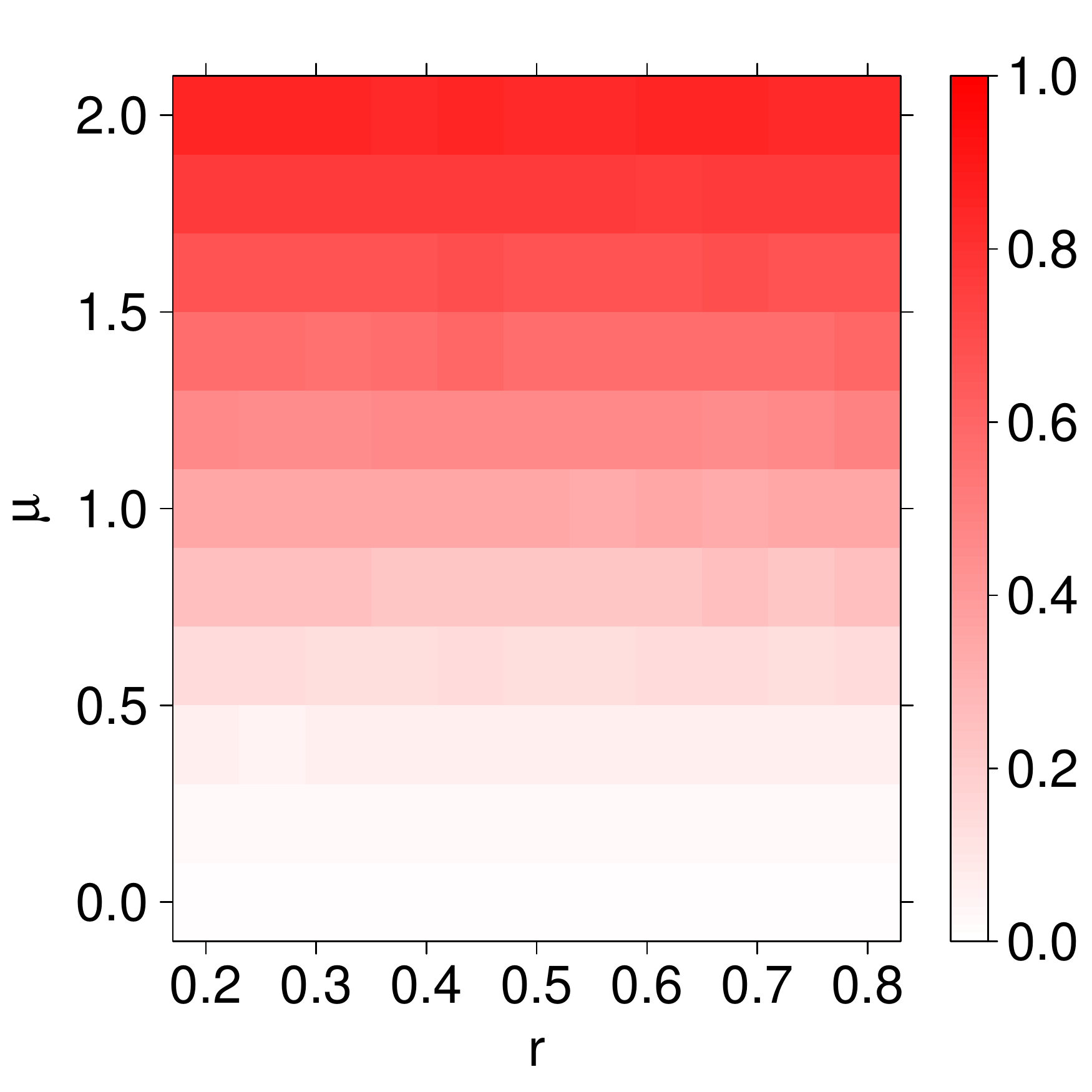}
		\caption{KM}
	\end{subfigure}
	\begin{subfigure}[h]{0.135\textwidth}
		\centering
		\includegraphics[width=\textwidth]{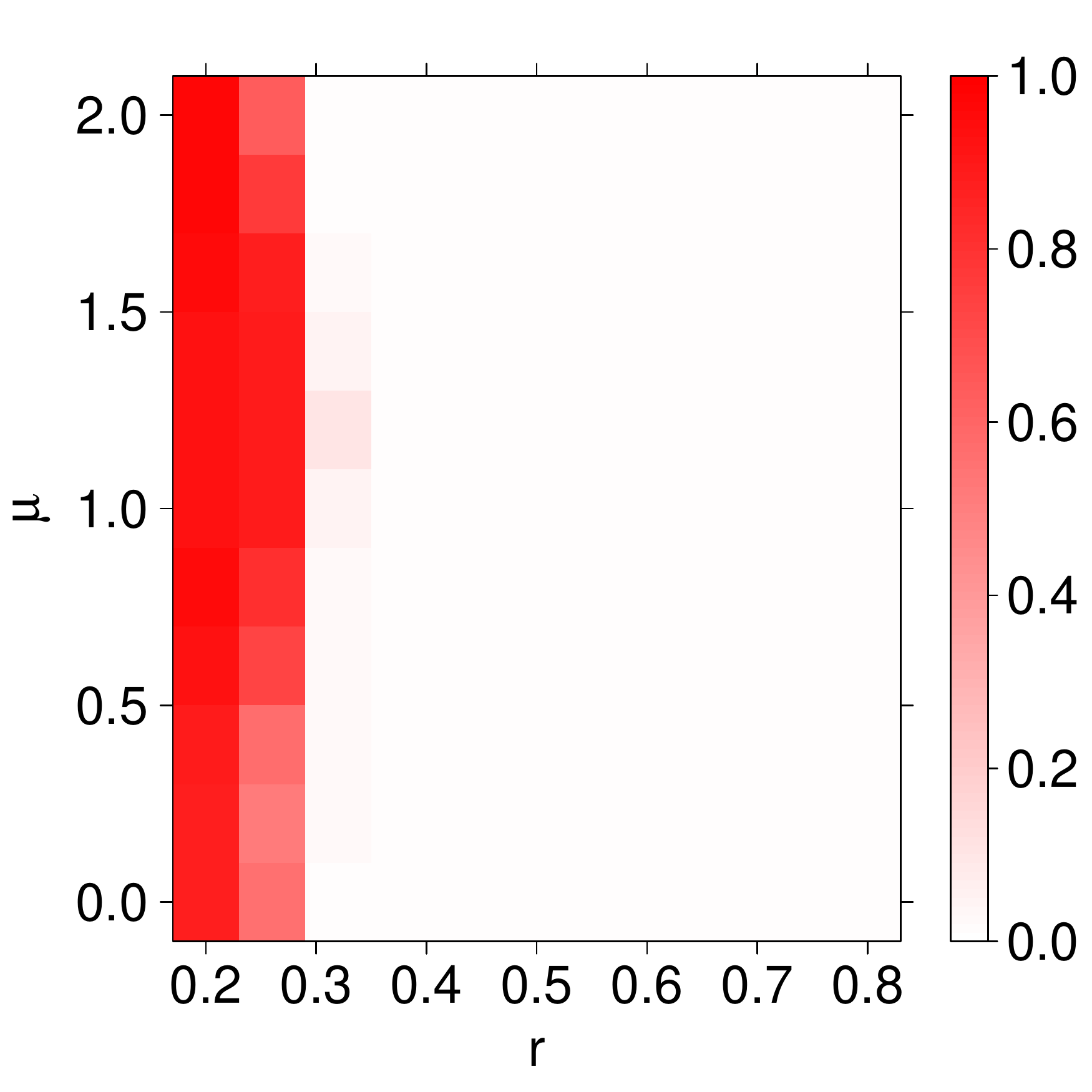}
		\caption{CASC}
	\end{subfigure}
	\begin{subfigure}[h]{0.135\textwidth}
		\centering
		\includegraphics[width=\textwidth]{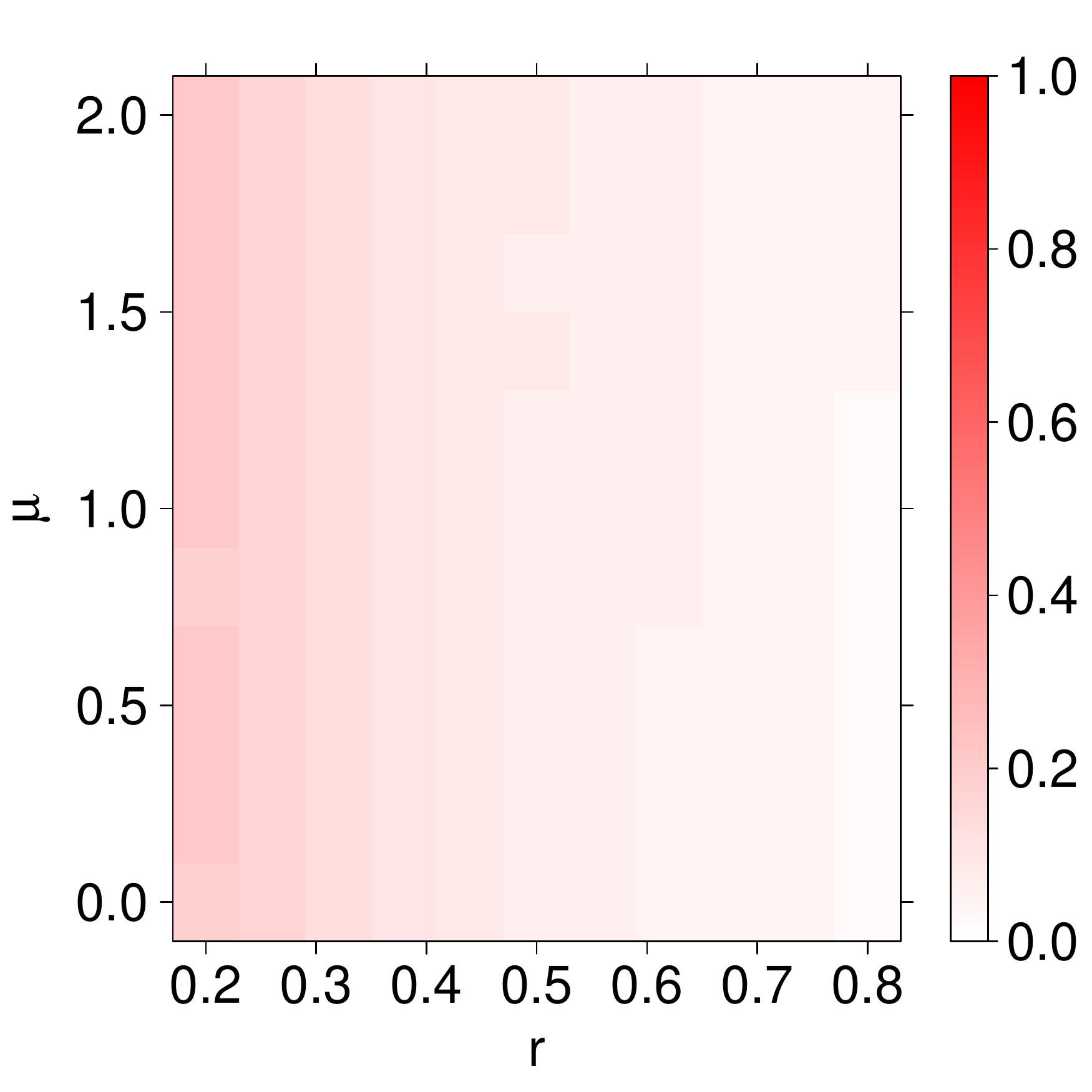}
		\caption{CESNA}
	\end{subfigure}
	\begin{subfigure}[h]{0.135\textwidth}
		\centering
		\includegraphics[width=\textwidth]{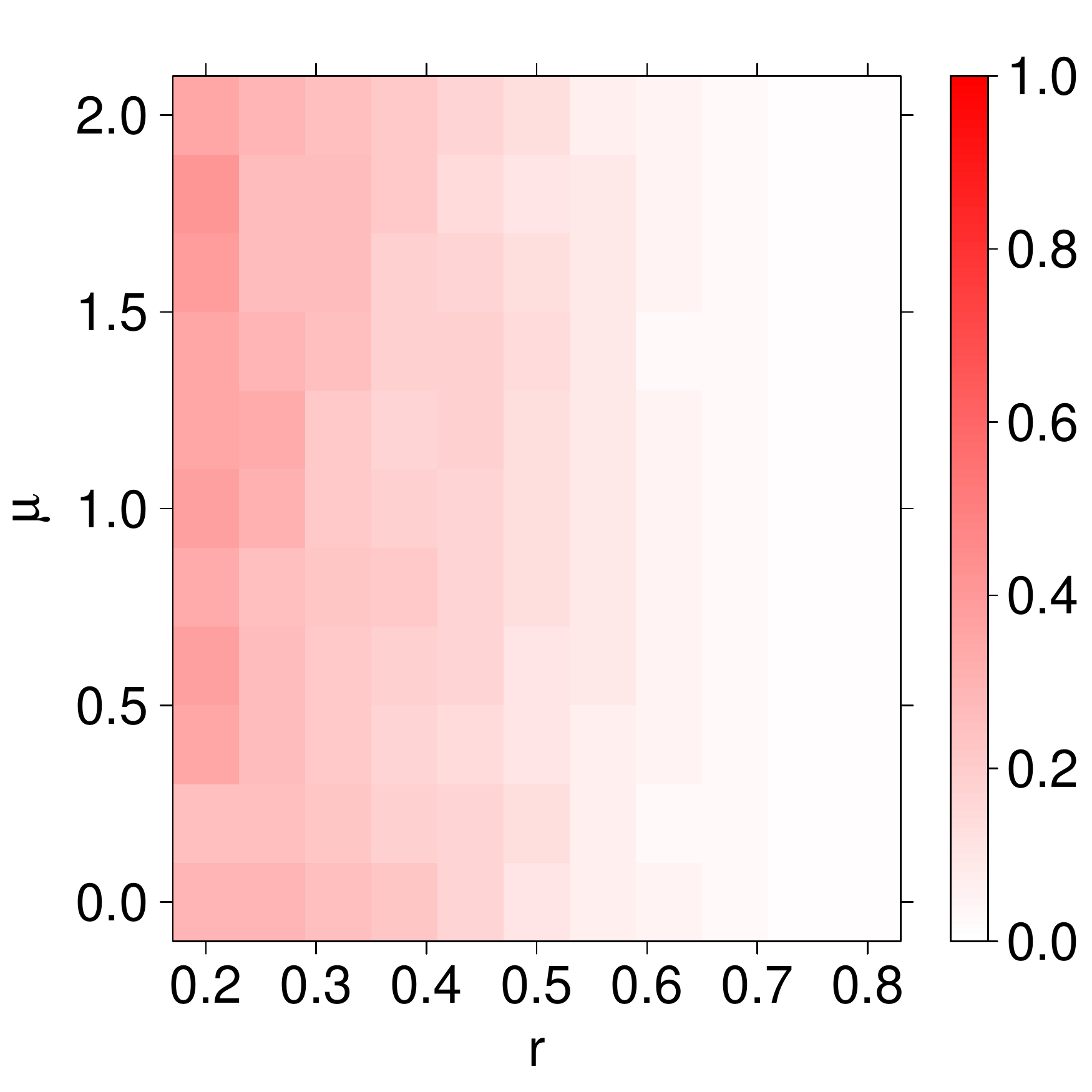}
		\caption{BAGC}
	\end{subfigure}
	\caption{Performance of different methods measured by normalized mutual information as a function of $r$ (out-in probability ratio) and $\mu$ (feature signal strength).}
	\label{fig_1}
\end{figure}


For each node $i$, we generate $p = 2$ features, with one ``signal'' feature related to the community structure and one ``noise'' feature whose distribution is the same for all nodes.  The ``signal'' feature follows the distribution $N(\mu, 1)$ for nodes in community 1 and $N(-\mu, 1)$ for nodes in community 2, with $\mu$ varying from $0.5$ to $2$ (larger $\mu$ corresponds to stronger signal).  For use with CESNA, which only allows categorical node features, we discretize the continuous node features by partitioning the real line into 20 bins using the $0.05, 0.1, \dots, 0.95$-th quantiles. 
For the JCDC, based on the study of the tuning parameters in the Supplemental Material, we use $\alpha=1$ and compare two values of $w_n$, $w_n=1.5$ and $w_n=5$.
Finally, agreement between the estimated communities and the true community labels is measured by normalized mutual information, a measure commonly used in the network literature which ranges between 0 (random guessing) and 1 (perfect agreement). For each configuration, we repeat the experiments 30 times, and record the average NMI over 30 replications. 

Figure \ref{fig_1} shows the heatmaps of average NMI for all methods under these settings, as a function of $r$ and $\mu$.  As one would expect, the performance of spectral clustering (c), which uses only the network information, is only affected by $r$ (the larger $r$ is, the harder the problem), and the performance of $K$-means (d), which uses only the features, is only affected by $\mu$ (the larger $\mu$ is, the easier the problem).   JCDC is able to take advantage of both network and feature information by estimating the coefficients $\beta$ from data, and its performance only deteriorates when neither is informative.   The  informative features are more helpful with a larger value of $w$ (a), and conversely uninformative features affect perfomance slightly more with a lower value of $w$ (b),  but this effect is not strong.  CASC (e)  appears to inherit the sharp phase transition from spectral clustering, which forms the basis of CASC;  the sharp transition is perhaps due to different community sizes and hub nodes, which are both challenging to spectral clustering;  CESNA (f) and BAGC (g) do not perform as well overall, with BAGC often clustering all the hub nodes into one community.

\section{Data applications}
\subsection{The world trade network}


The world trade network \cite{de2011Pajek} connects 80 countries based on the amount of trade of metal manufactures between them in 1994, or when not available for that year,  in 1993 or 1995.   Nodes are countries and edges represent positive amount of import and/or export between the countries.  Each country also has three categorical features: the continent  (Africa, Asia, Europe, N.\ America, S.\ America, and Oceania), the country's structural position in the world system in 1980 (core, strong semi-periphery, weak semi-periphery, periphery) and in 1994 (core, semi-periphery, periphery). Figures \ref{worldtrade} (a) to (c) show the adjacency matrix rearranged by sorting the nodes by each of the features.  The partition by continent (Figure \ref{worldtrade}(a)) clearly shows community structure, whereas the other two features show hubs (core status countries trade with everyone), and no assortative community structure.   We will thus compare partitions found by all the competing methods to the continents, and omit the three Oceania countries from further analysis because no method is likely to detect such a small community.    The two world position variables ('80 and '94) will be used as features, treated as ordinal variables. 

\setcounter{subfigure}{0}
\begin{figure}[H]
	\centering
	\begin{subfigure}[b]{0.19\textwidth}
	\includegraphics[width=\textwidth]{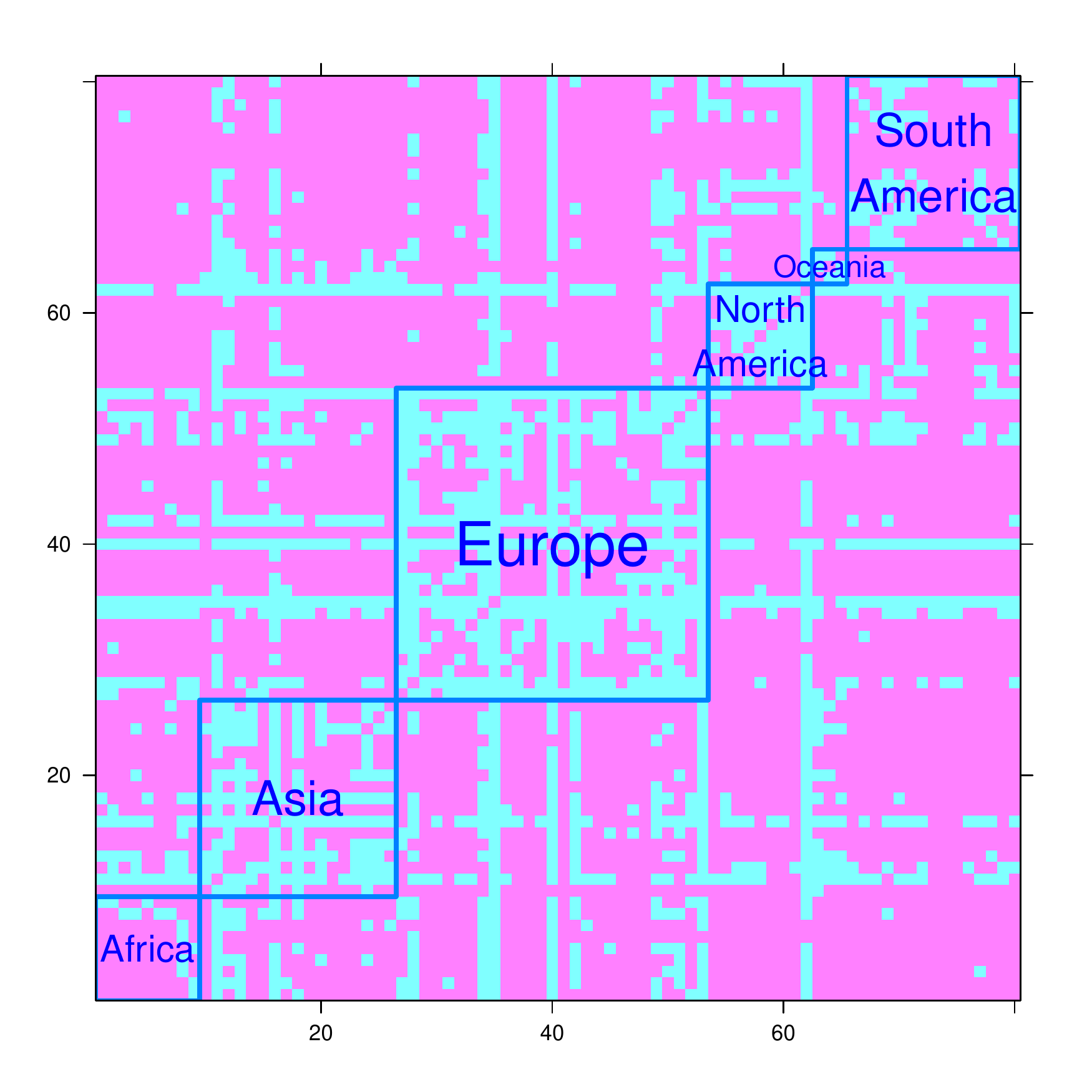}
	\caption{$A$ by continent}
	\end{subfigure}
	\begin{subfigure}[b]{0.19\textwidth}
	\includegraphics[width=\textwidth]{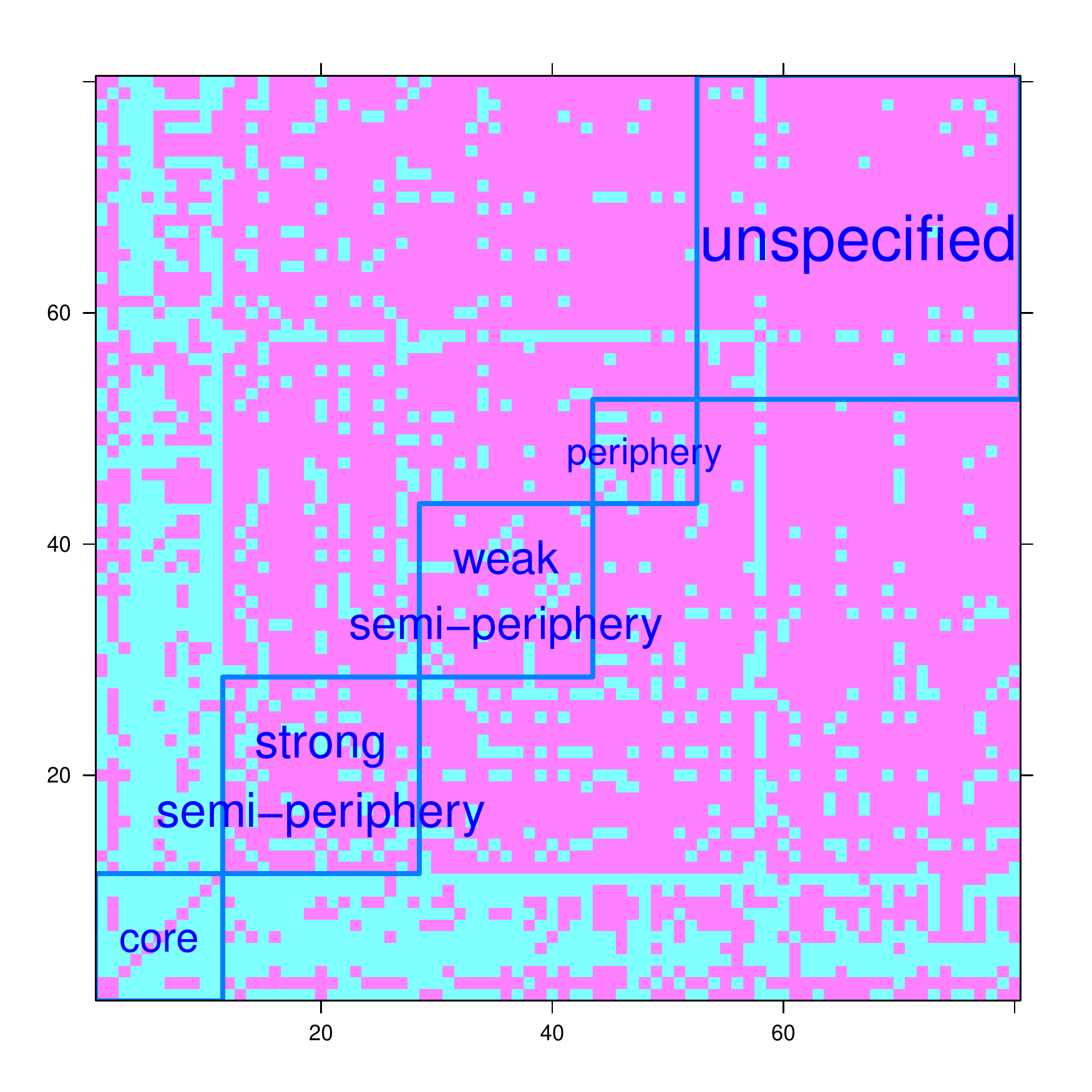}
	\caption{$A$ by position '80}
	\end{subfigure}
	\begin{subfigure}[b]{0.19\textwidth}
	\includegraphics[width=\textwidth]{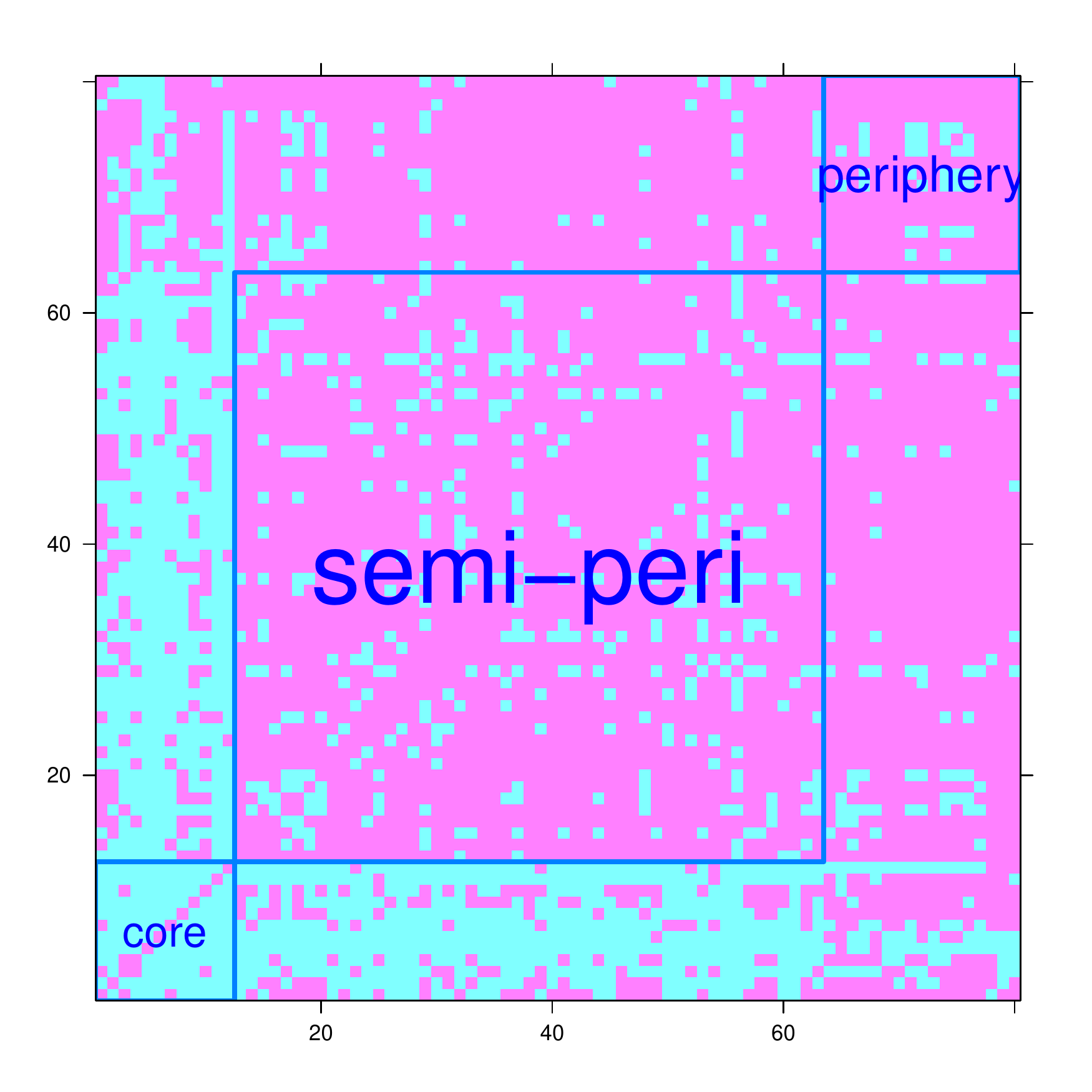}
	\caption{$A$ by position '94}
	\end{subfigure}
	\begin{subfigure}[b]{0.16\textwidth}
		\includegraphics[width=\textwidth]{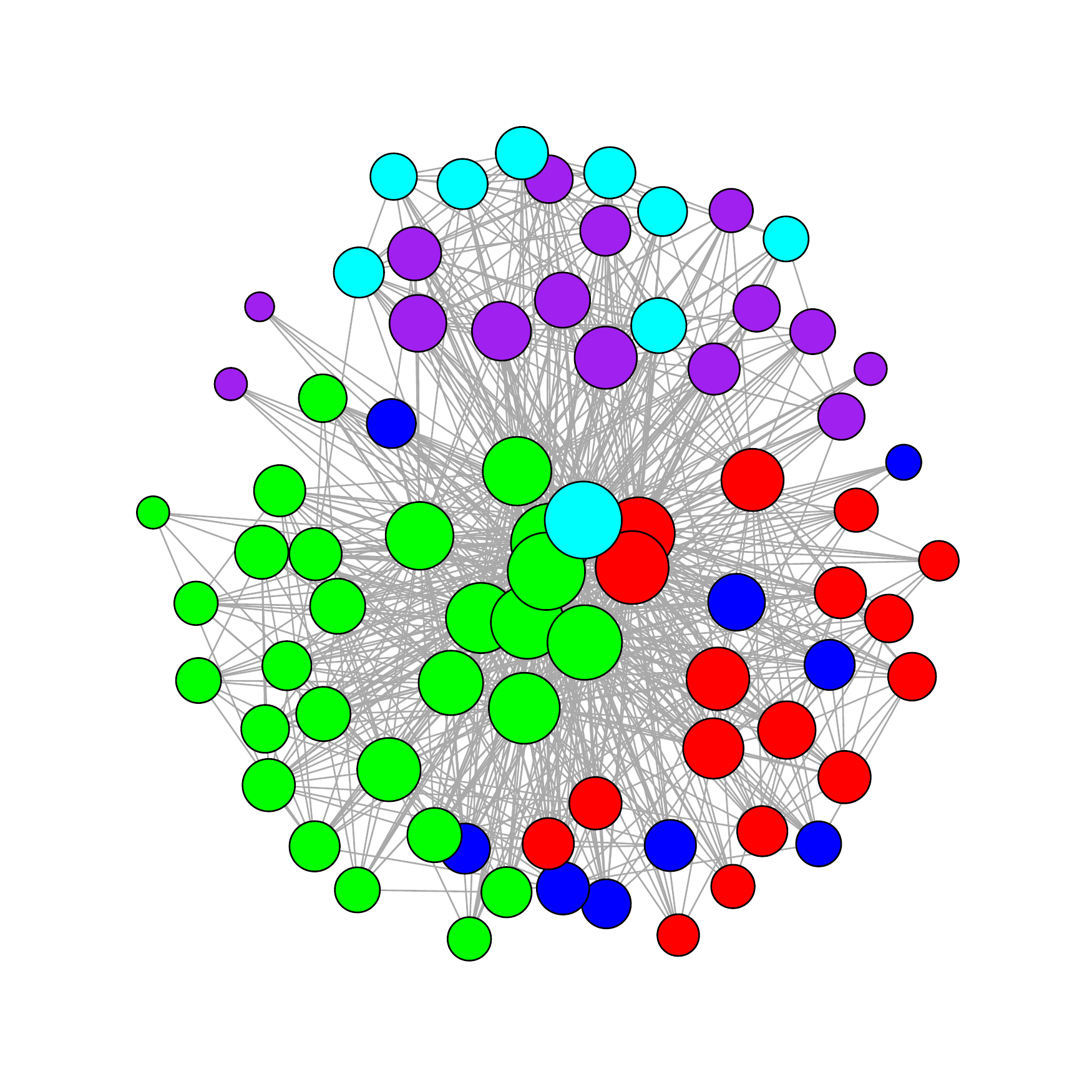}
		\caption{Continent\\\ }
	\end{subfigure}
	\begin{subfigure}[b]{0.16\textwidth}
		\includegraphics[width=\textwidth]{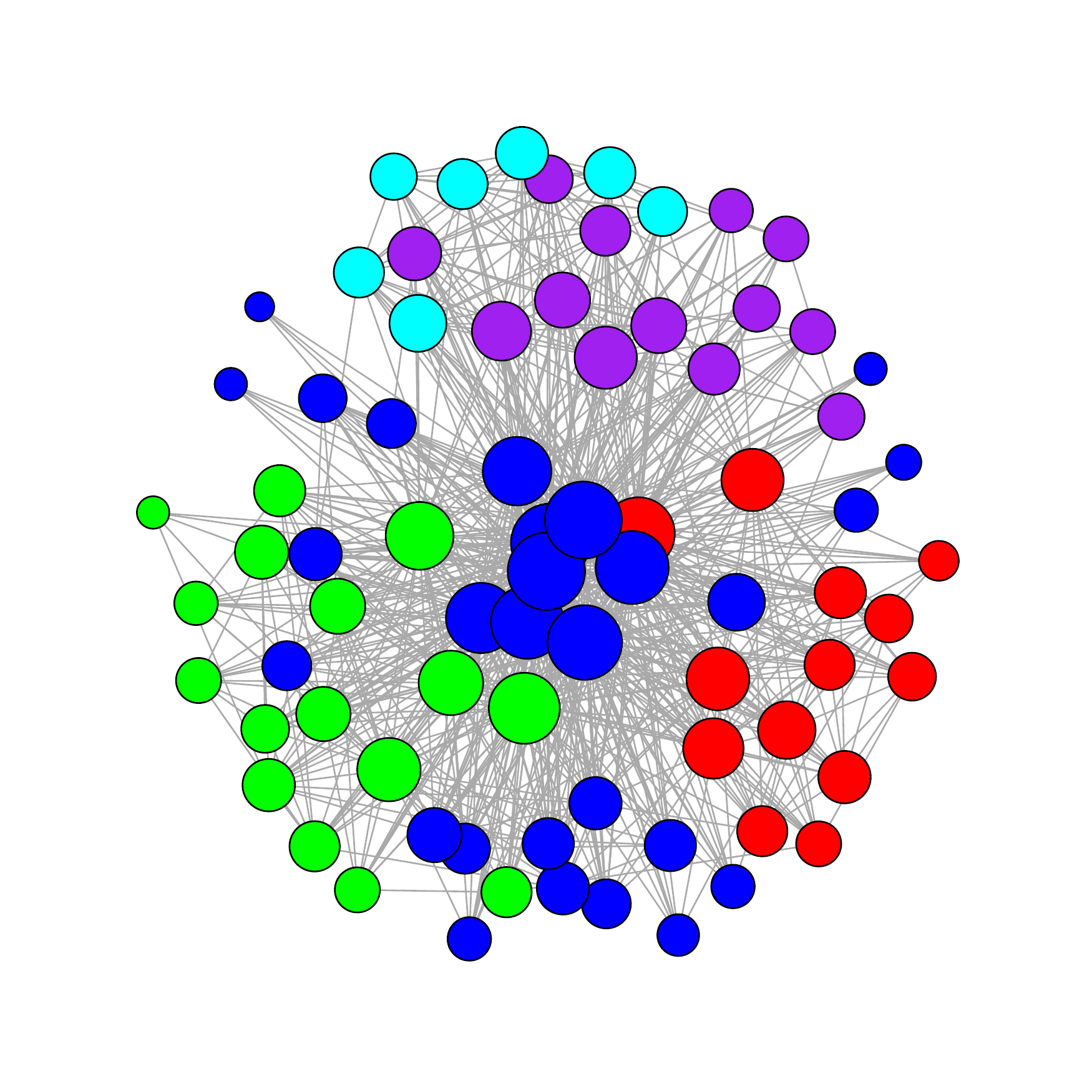}
		\caption{JCDC, $w_n=5$\\NMI=0.54}
	\end{subfigure}
	\begin{subfigure}[b]{0.16\textwidth}
			\includegraphics[width=\textwidth]{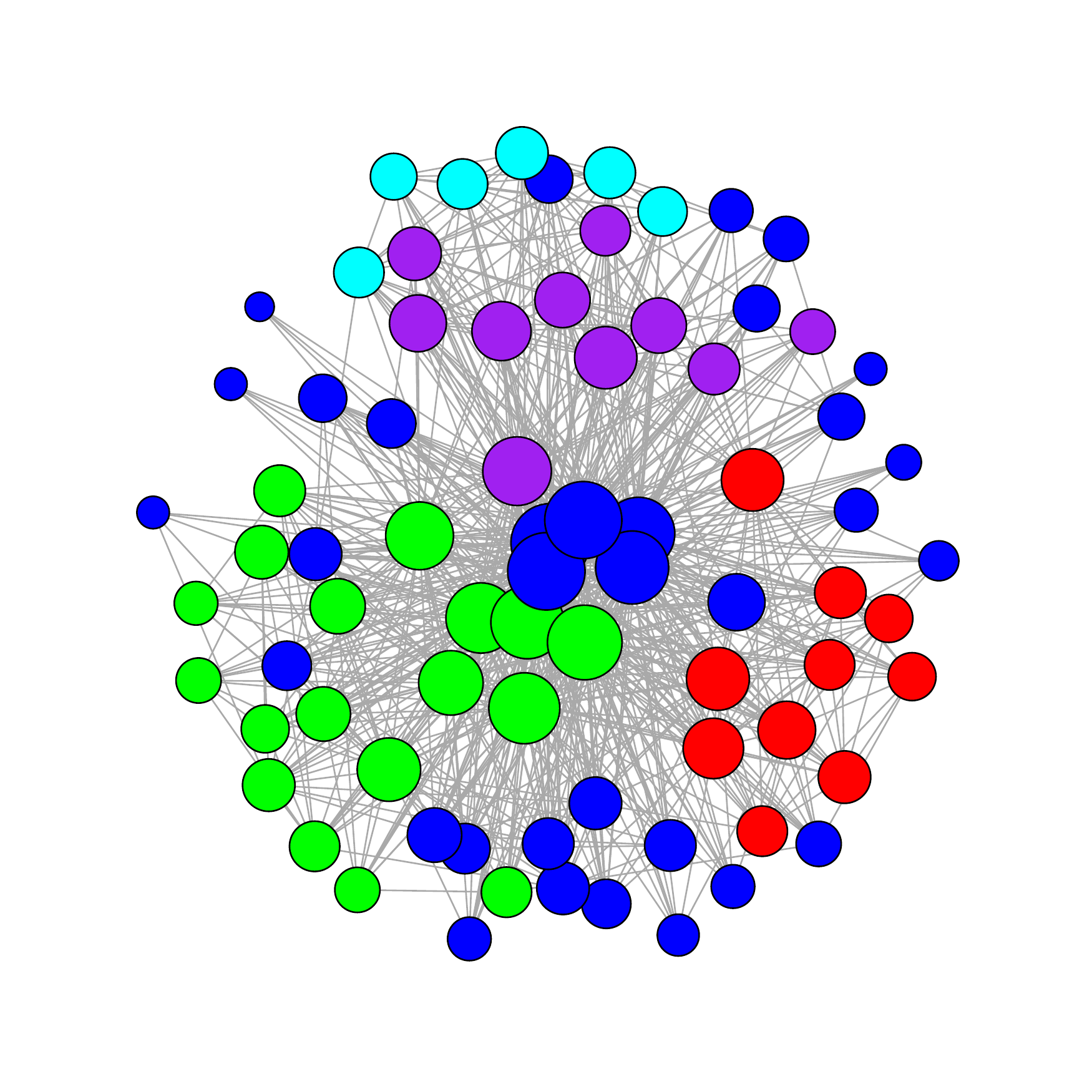}
			\caption{JCDC, $w_n=1.5$\\NMI=0.50}
		\end{subfigure}
	\begin{subfigure}[b]{0.16\textwidth}
		\includegraphics[width=\textwidth]{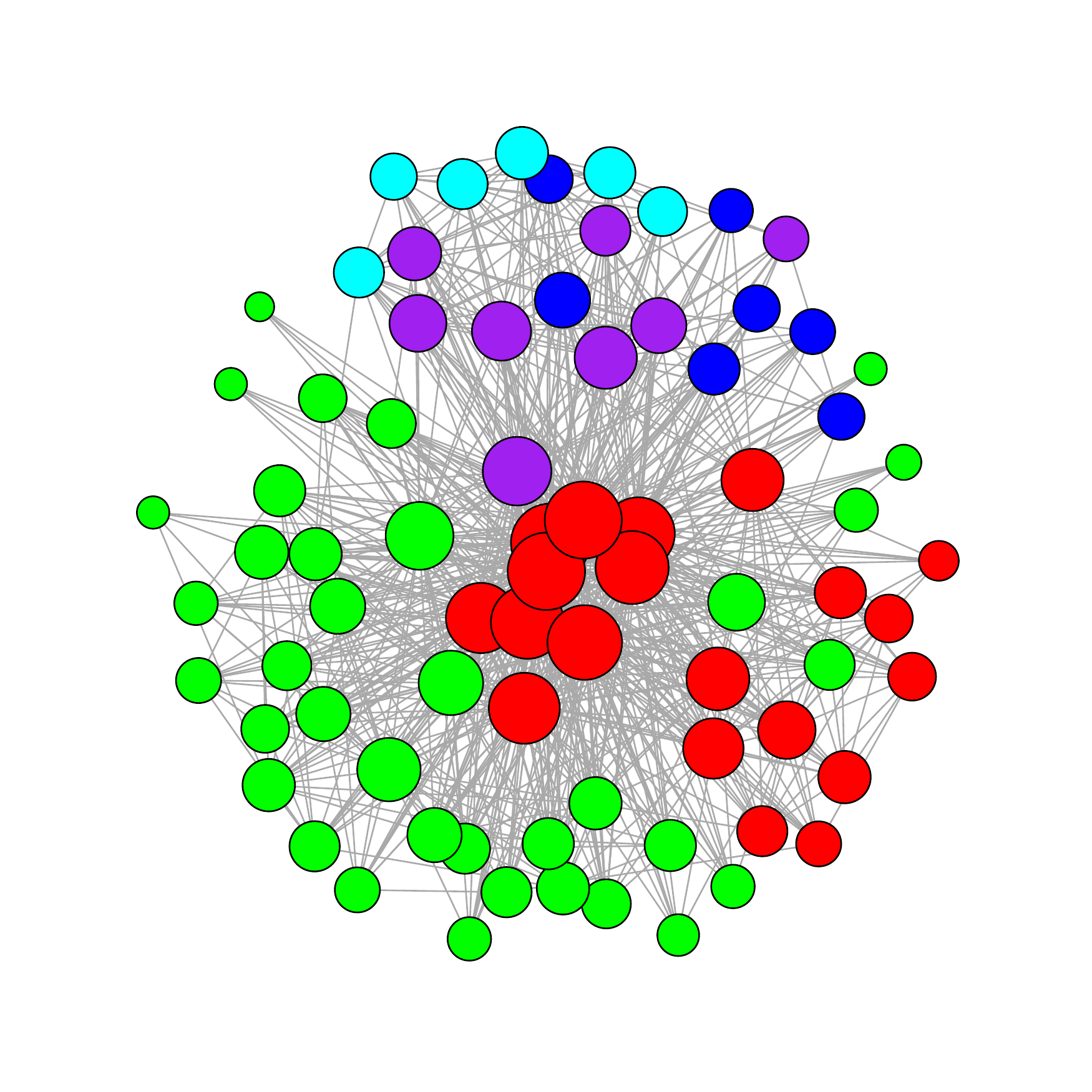}
		\caption{SC\\NMI=0.47}
	\end{subfigure}
	\begin{subfigure}[b]{0.16\textwidth}
		\includegraphics[width=\textwidth]{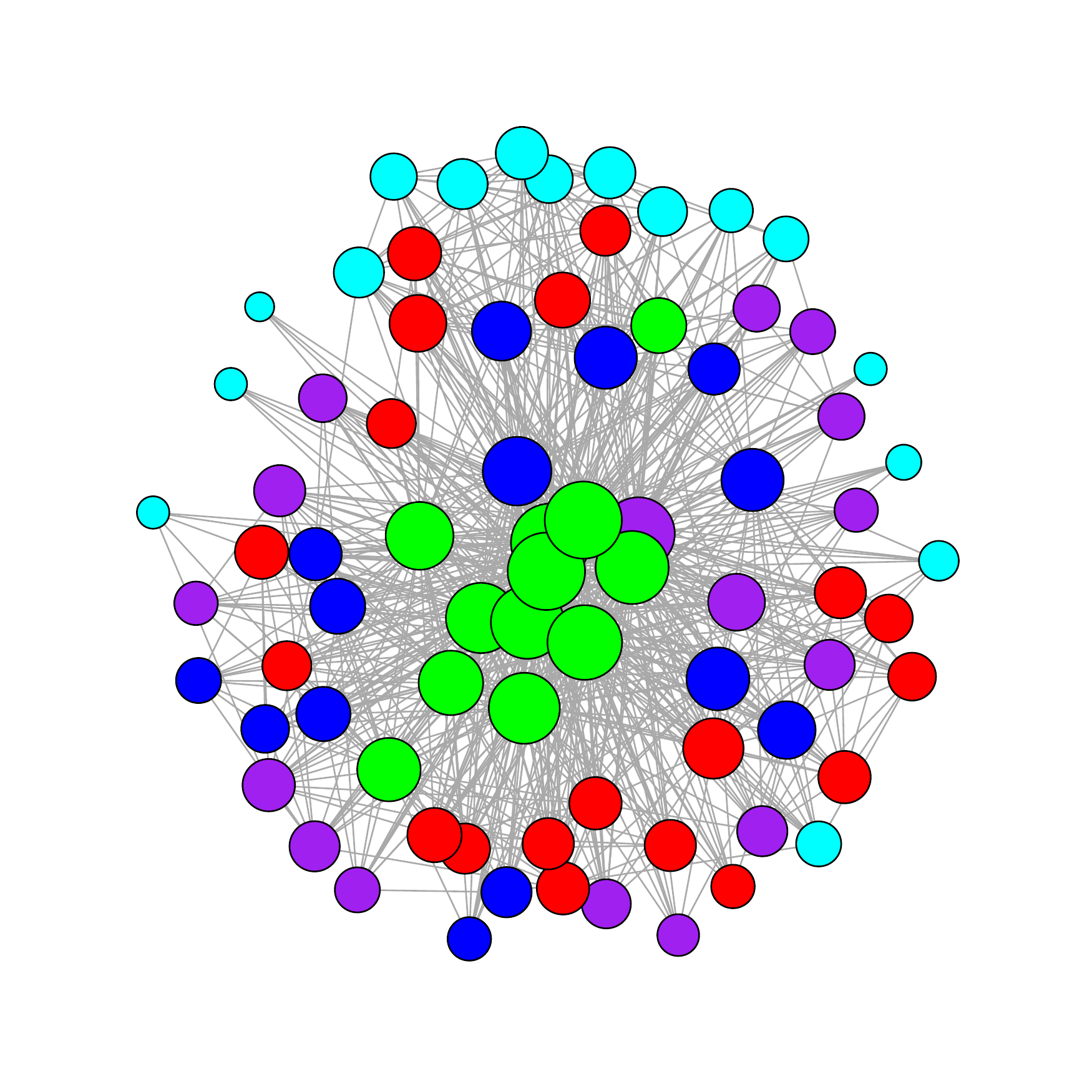}
		\caption{KM\\NMI=0.25}
	\end{subfigure}
	\begin{subfigure}[b]{0.16\textwidth}
		\includegraphics[width=\textwidth]{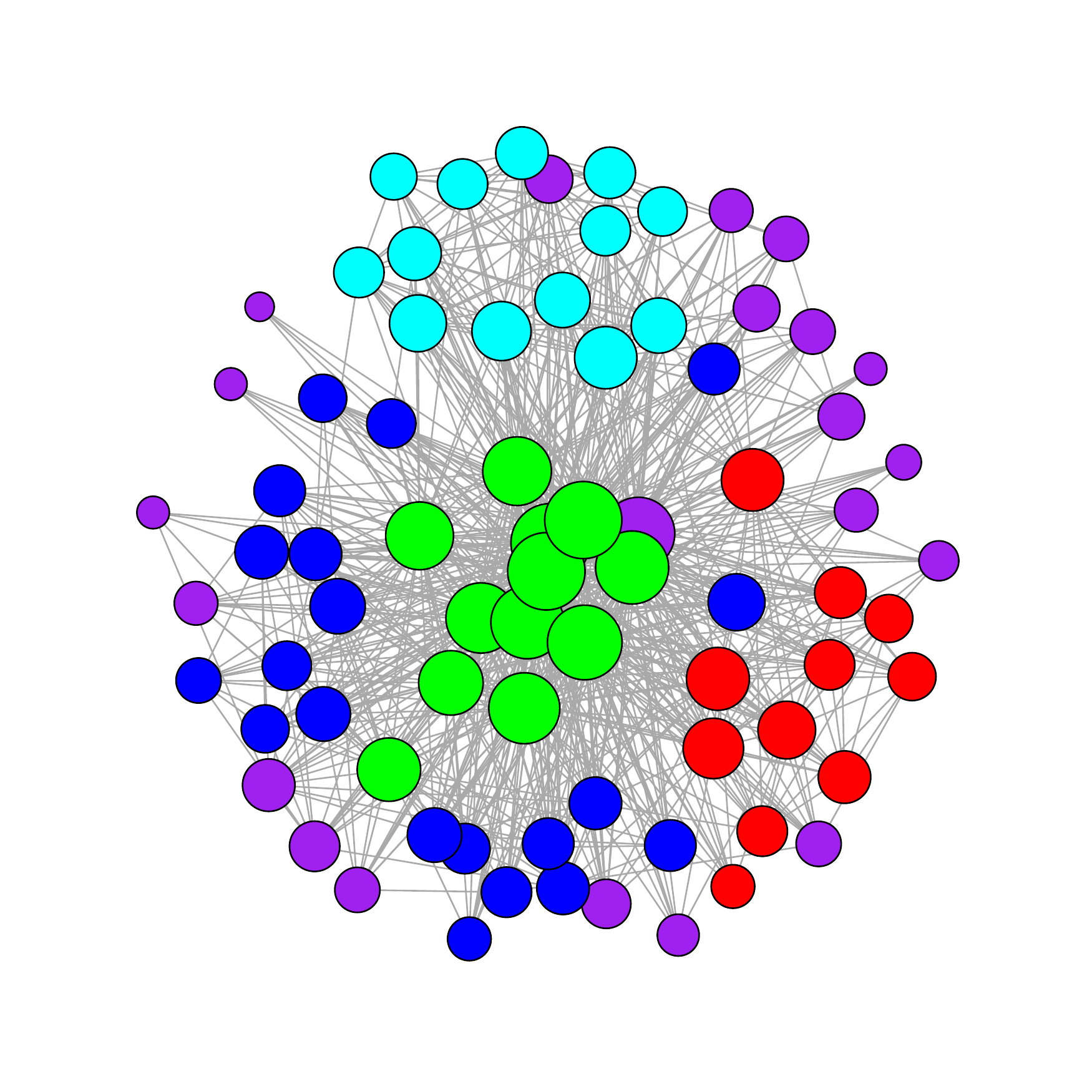}
		\caption{CASC\\NMI=0.39}
	\end{subfigure}
	\begin{subfigure}[b]{0.16\textwidth}
		\includegraphics[width=\textwidth]{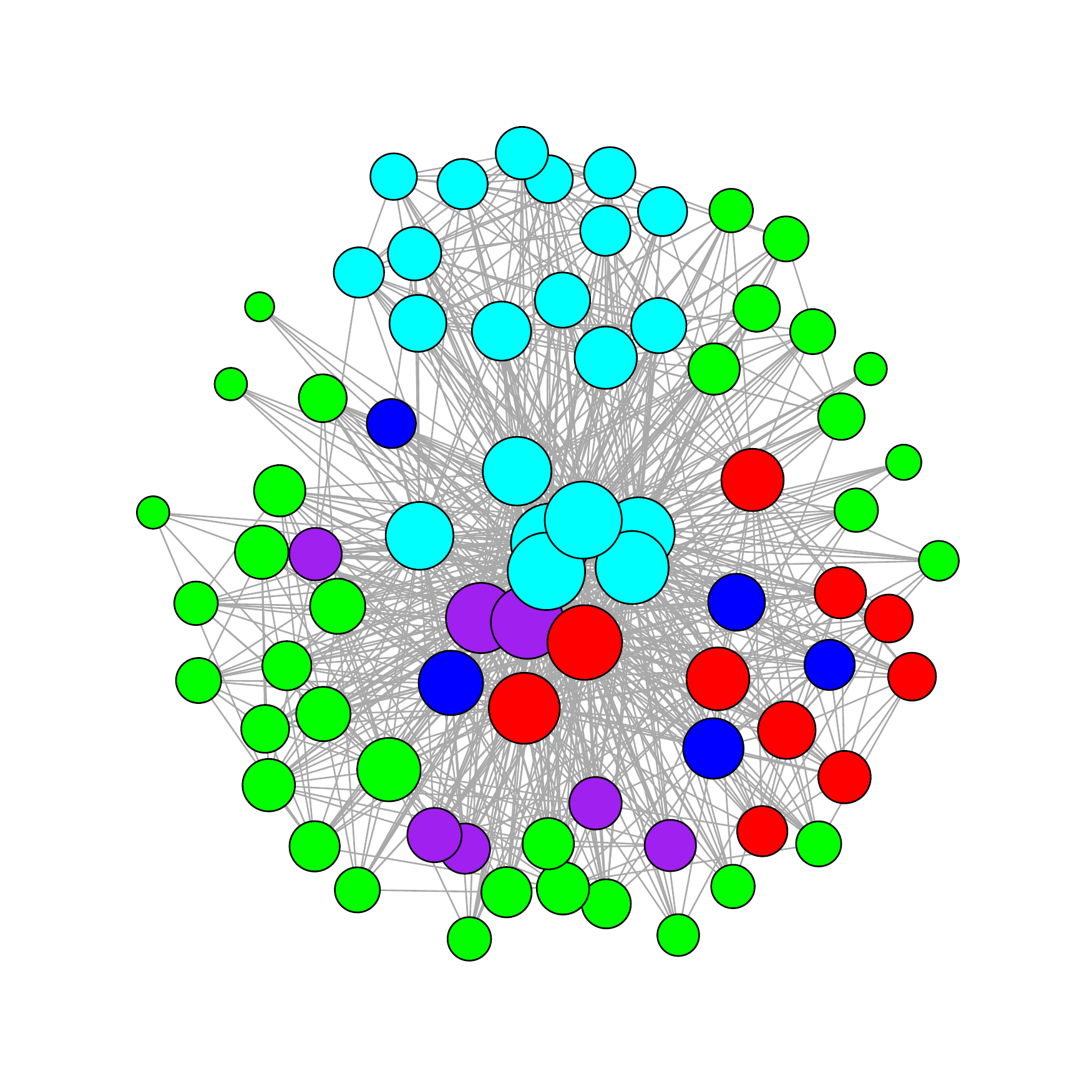}
		\caption{CESNA\\NMI=0.26}
	\end{subfigure}
	\begin{subfigure}[b]{0.16\textwidth}
		\includegraphics[width=\textwidth]{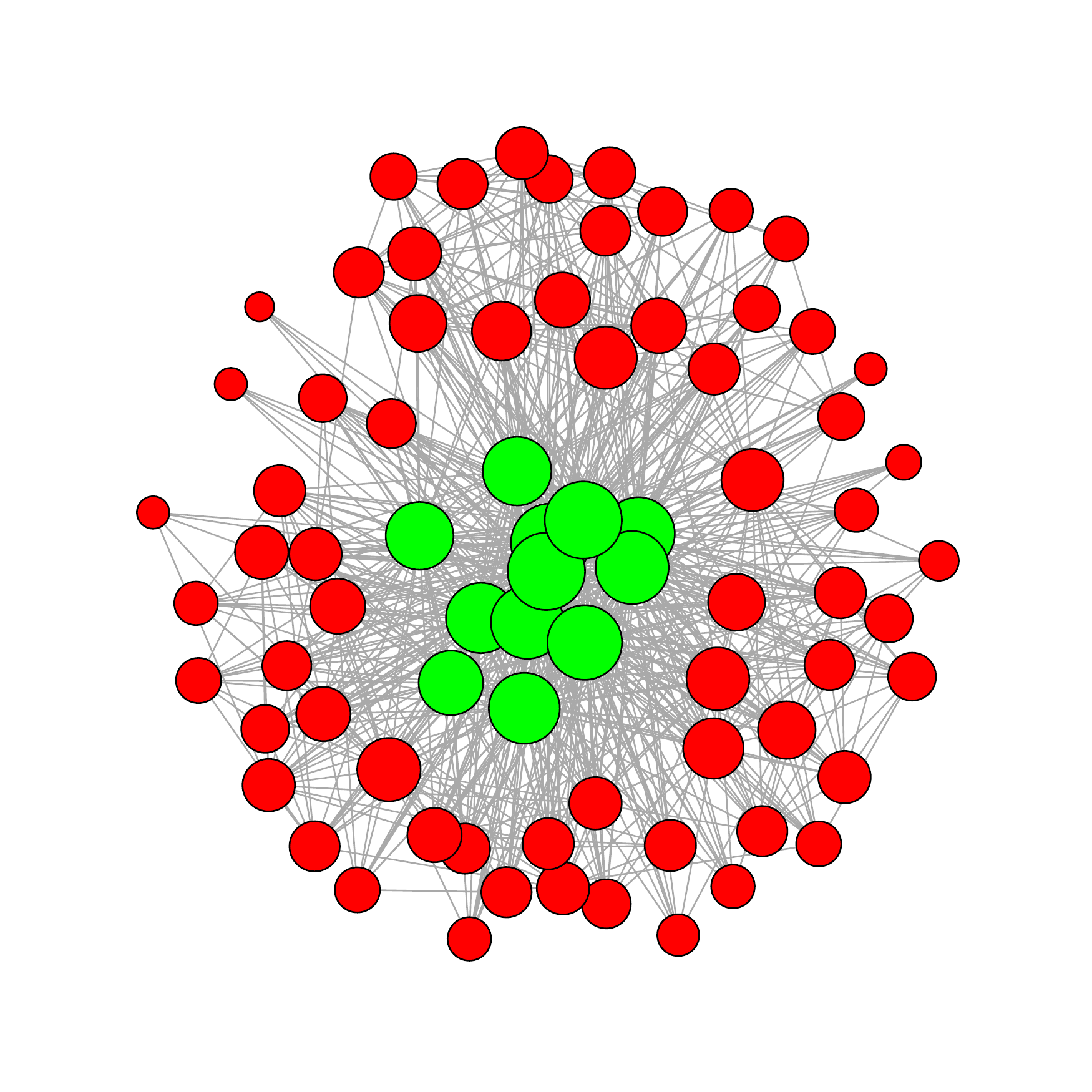}
		\caption{BAGC\\NMI=0.11}
	\end{subfigure}

	\caption{(a)-(c): the adjacency matrix ordered by different node features; (d) network with nodes colored by continent (taken as ground truth); blue is Africa, red is Asia, green is Europe, cyan is N.\ America and purple is S.\ America.  (e)-(k) community detection results from different methods;  colors are mated to (d) in the best way possible. }\label{worldtrade}
\end{figure}

The results for all methods are shown in Figure \ref{worldtrade}, along with NMI values comparing the detected partition to the continents.  All methods were run with the true value $K=5$.

\begin{table}[h]
	\centering
\caption{Feature coefficients $\hat{\beta}_k$ estimated by JCDC with $w=5$. Best match is determined by majority vote.}
 \begin{tabular}{llll}
Community	& Best match	& Position '80 & Position '94\\\hline
blue	& Europe 	& 0.000	& 0.143\\
red  	& Asia		& 0.314	& 0.127\\
green	& Europe 	& 0.017	& 0.204\\
cyan 	& N. America & 0.107 & 0.000\\
purple & S. America & 0.121	& 0.000\\\hline
\end{tabular}
\label{table:worldtrade}
\end{table}

The result of spectral clustering agrees much better with the continents than that of $K$-means, indicating that the community structure in the adjacency matrix is closer to the continents that the structure contained in the node features.  JCDC obtains the highest NMI value, 
CASC performs similarly to spectral clustering, whereas CESNA and BAGC both fail to recover the continent partition. Note that no method was able to estimate Africa well, likely due to the disassortative nature of its trade seen in Figure \ref{worldtrade} (a).    Figure \ref{worldtrade} (e) indicates that JCDC estimated N. America, S. America and Asia with high accuracy, but split Europe into two communities, since it was run with $K=5$ and could not pick up Africa due to its disassortative structure.   Table \ref{table:worldtrade} contains the estimated feature coefficients, suggesting that in 1980 the ``world position'' had the most influence on the connections formed by Asian countries, whereas in 1994 world position mattered most in Europe.





\subsection{The lawyer friendship network}


The second dataset we consider is a friendship network of 71 lawyers in a New England corporate law firm  \cite{lazega2001collegial}.  
Seven node features are available: status (partner or associate), gender, office location (Boston, Hartford, or Providence, a very small office with only two non-isolated nodes), years with the firm, age, practice (litigation or corporate) and law school attended (Harvard, Yale, University of Connecticut, or other).  Categorical features with $M$ levels are represented by $M-1$ dummy indicator variables.  Figures \ref{lawyers} (a)-(g) show heatmap plots of the adjacency matrix with nodes sorted by each feature,  after eliminating six isolated nodes.  Partition by status (Figure \ref{lawyers}(a)) shows a strong assortative structure, and so does partition by office (Figure \ref{lawyers}(c)) restricted to Boston and Hartford, but the small Providence office does not have any kind of structure.   Thus we chose the status partition as a reference point for comparisons, though other partitions are certainly also meaningful.  
\setcounter{subfigure}{0}
\begin{figure}[h]
	\centering
	\begin{subfigure}[b]{0.18\textwidth}
	\includegraphics[width=\textwidth]{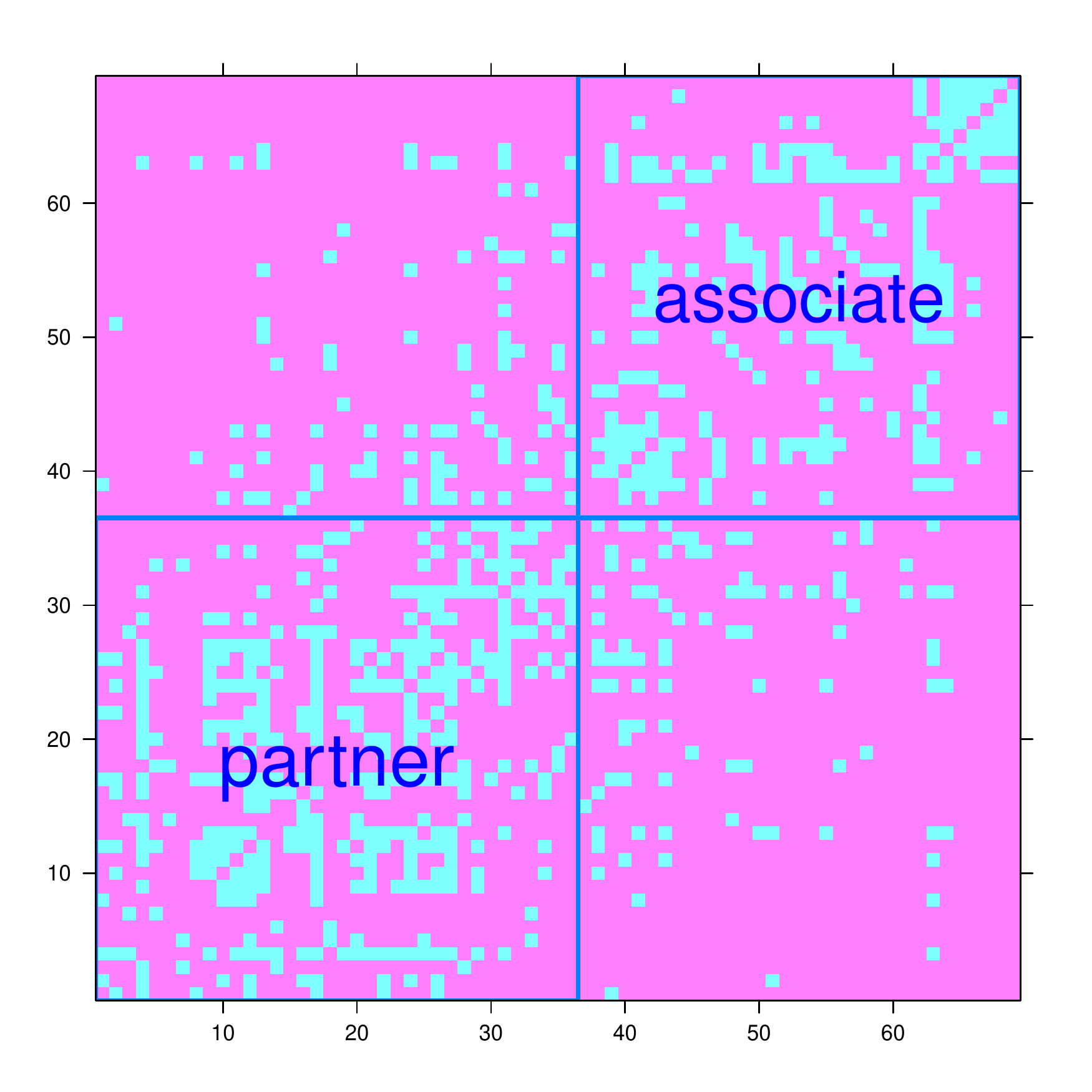} 
	\caption{$A$ by status}
	\end{subfigure}
	\begin{subfigure}[b]{0.18\textwidth}
	\includegraphics[width=\textwidth]{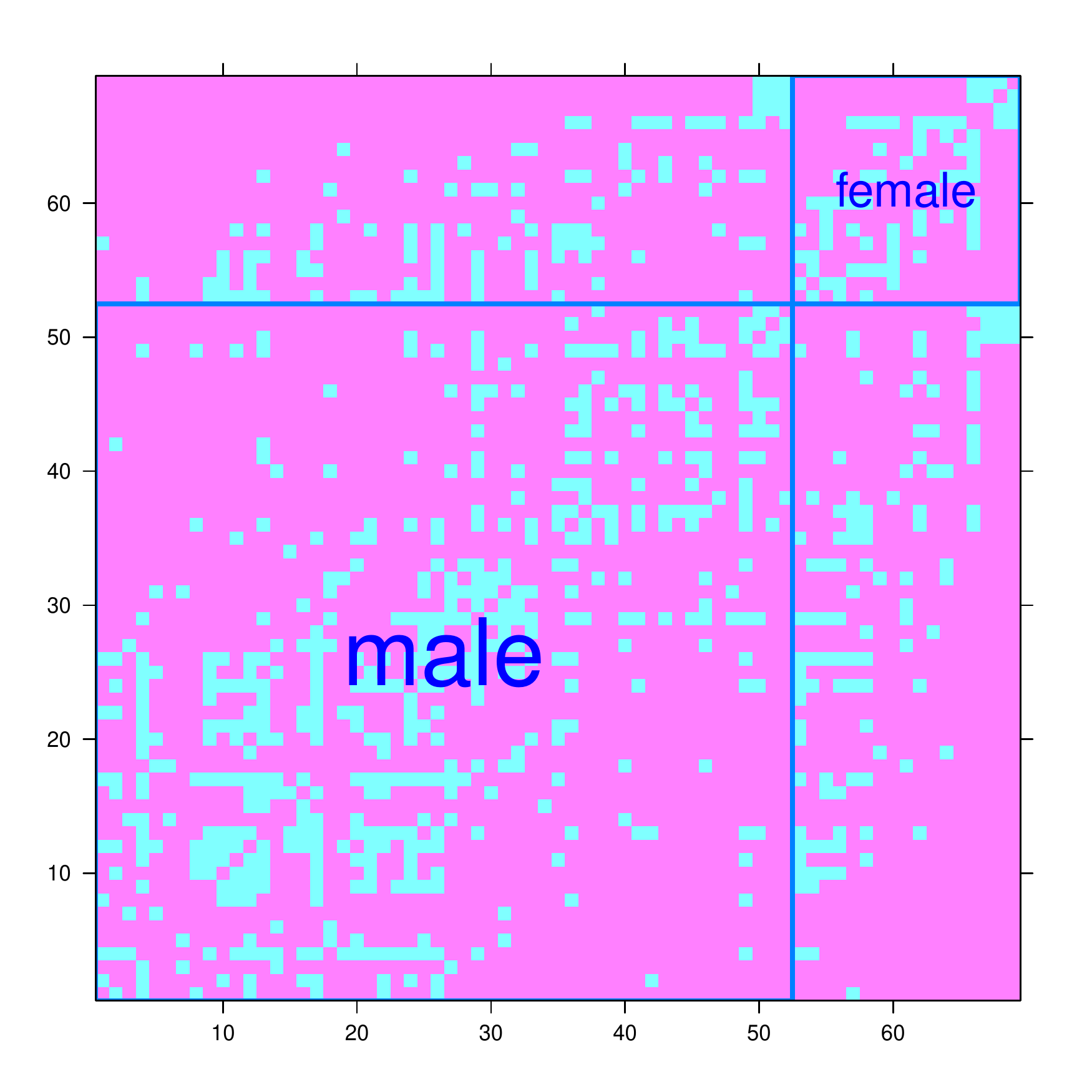}
	\caption{$A$ by gender}
	\end{subfigure}
	\begin{subfigure}[b]{0.18\textwidth}
	\includegraphics[width=\textwidth]{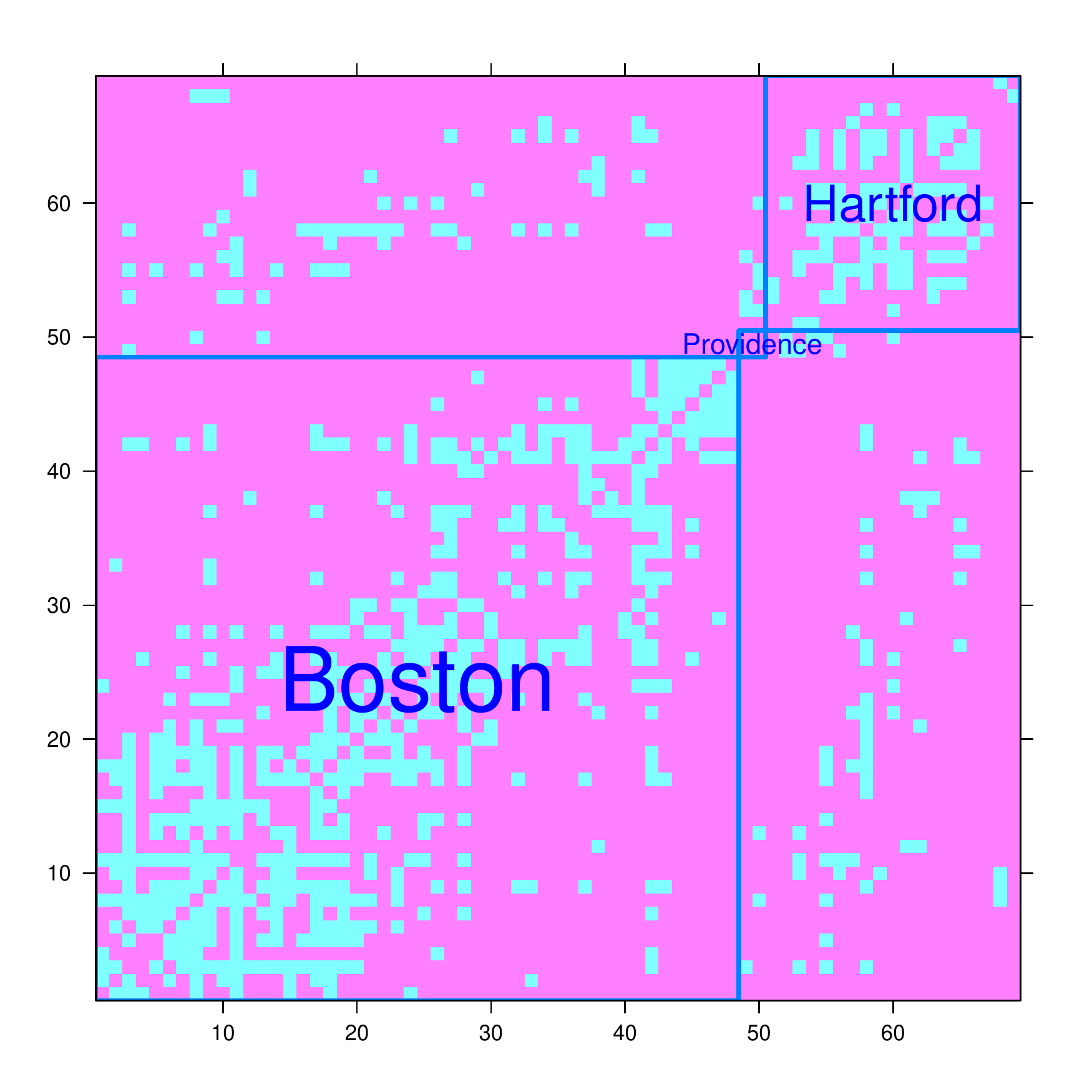} 
	\caption{$A$ by office}
	\end{subfigure}
	\begin{subfigure}[b]{0.18\textwidth}
	\includegraphics[width=\textwidth, page=4]{fig/lawyers_grouping.pdf}
	\caption{$A$ by years}
	\end{subfigure}
	\begin{subfigure}[b]{0.18\textwidth}
	\includegraphics[width=\textwidth, page=5]{fig/lawyers_grouping.pdf} 
	\caption{$A$ by age}
	\end{subfigure}
	\begin{subfigure}[b]{0.18\textwidth}
	\includegraphics[width=\textwidth]{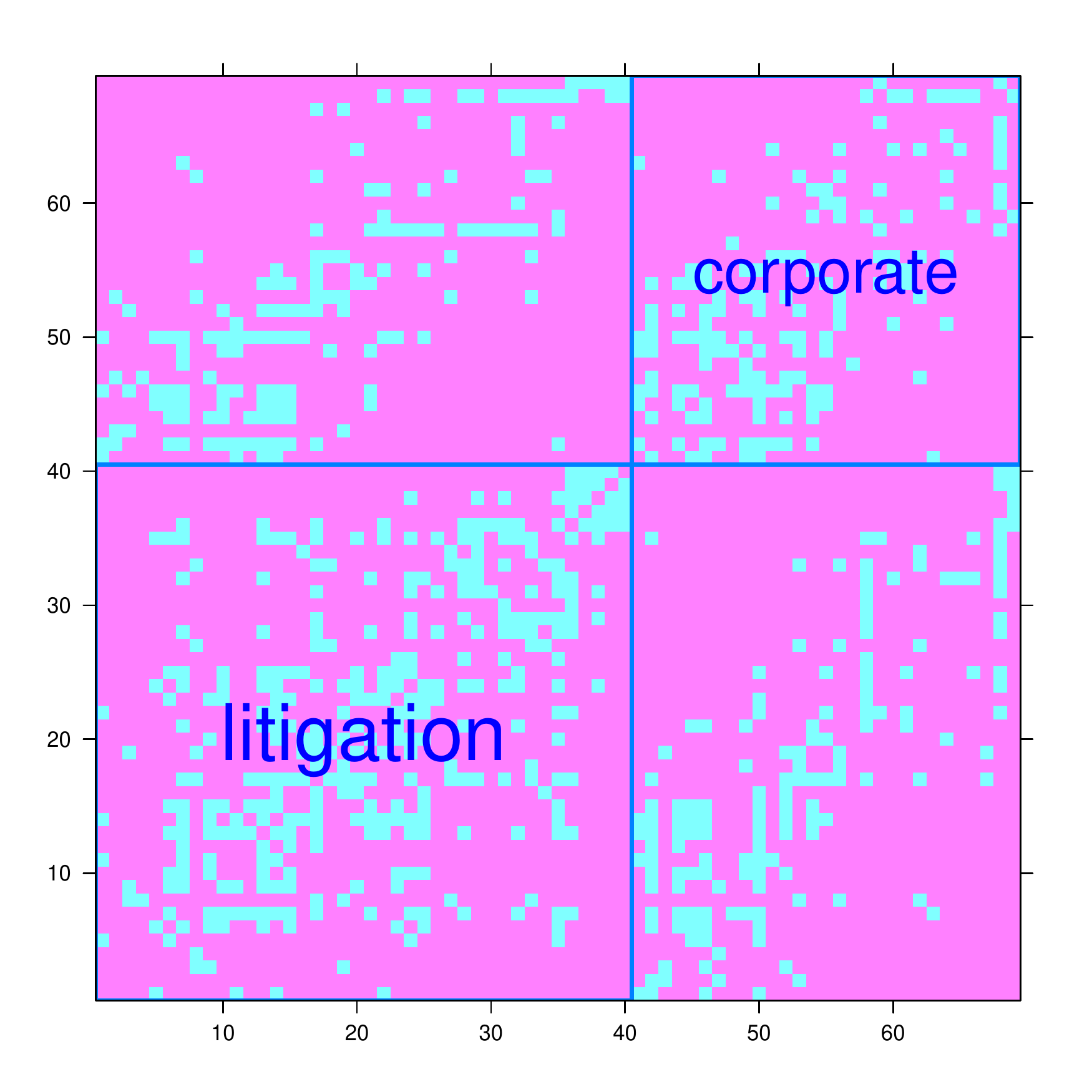} 
	\caption{$A$ by practice}
	\end{subfigure}
	\begin{subfigure}[b]{0.18\textwidth}
	\includegraphics[width=\textwidth]{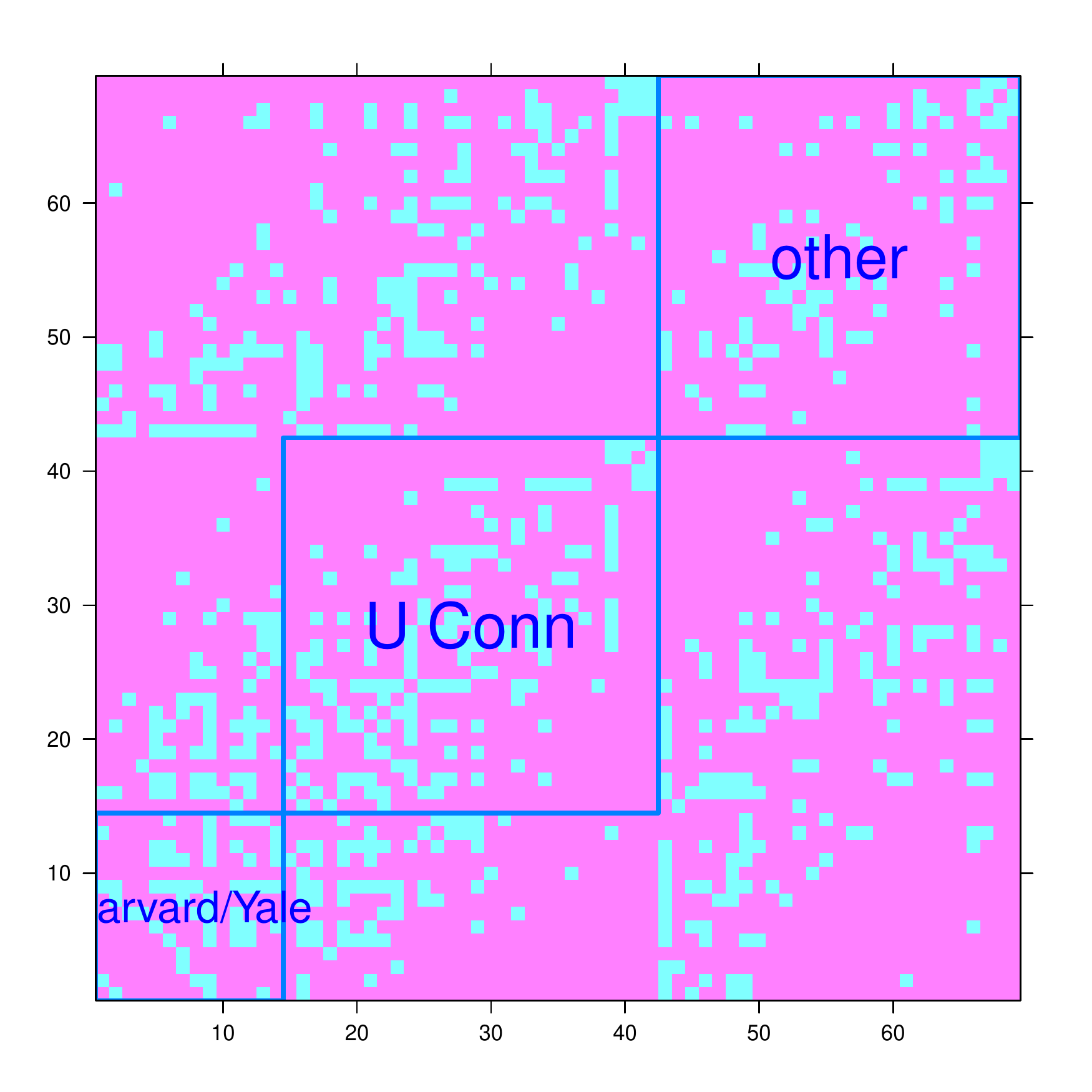} 
	\caption{$A$ by school}
	\end{subfigure}
	\begin{subfigure}[b]{0.18\textwidth}
		\includegraphics[width=\textwidth]{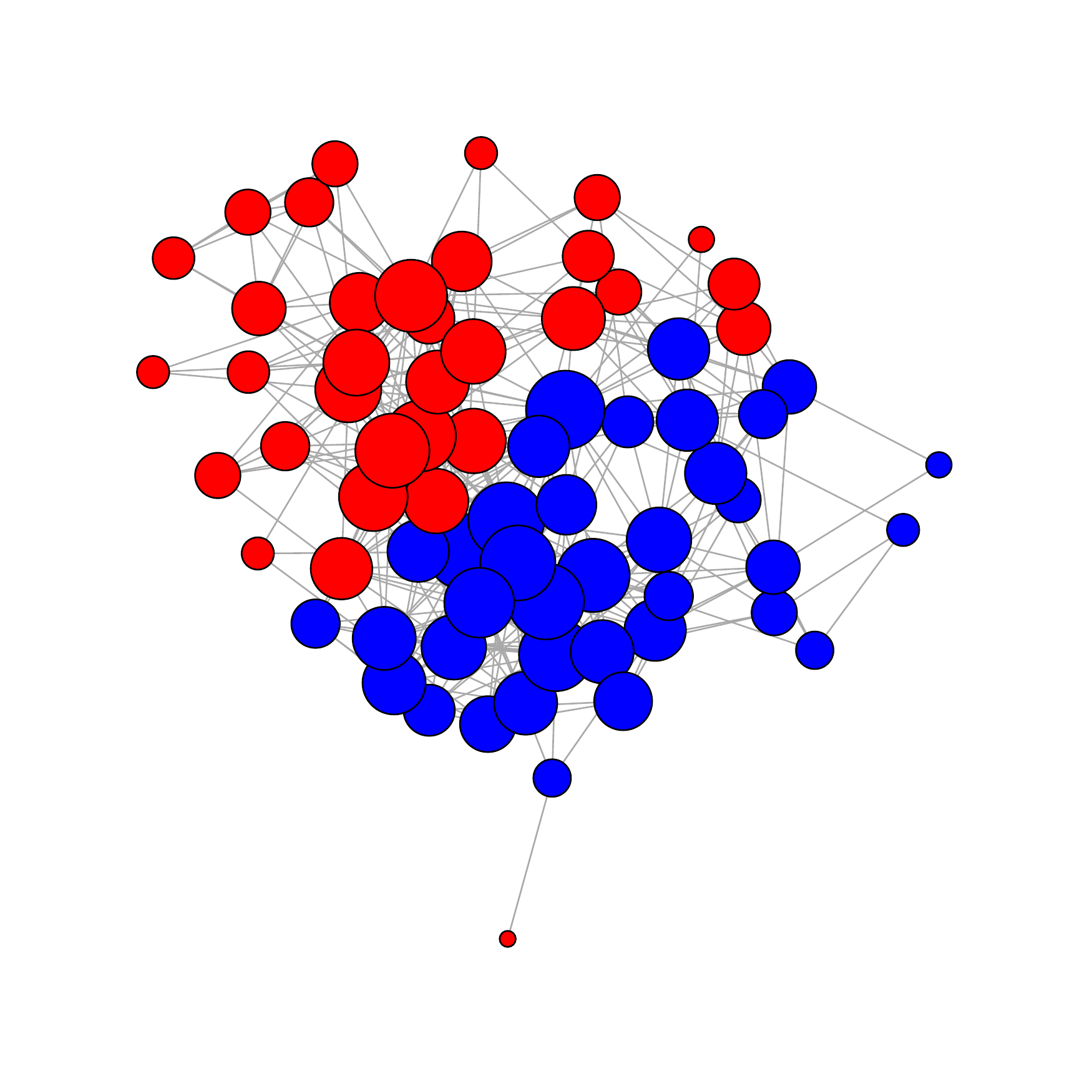}
		\caption{Status\\\ }
	\end{subfigure}
	\begin{subfigure}[b]{0.18\textwidth}
		\includegraphics[width=\textwidth]{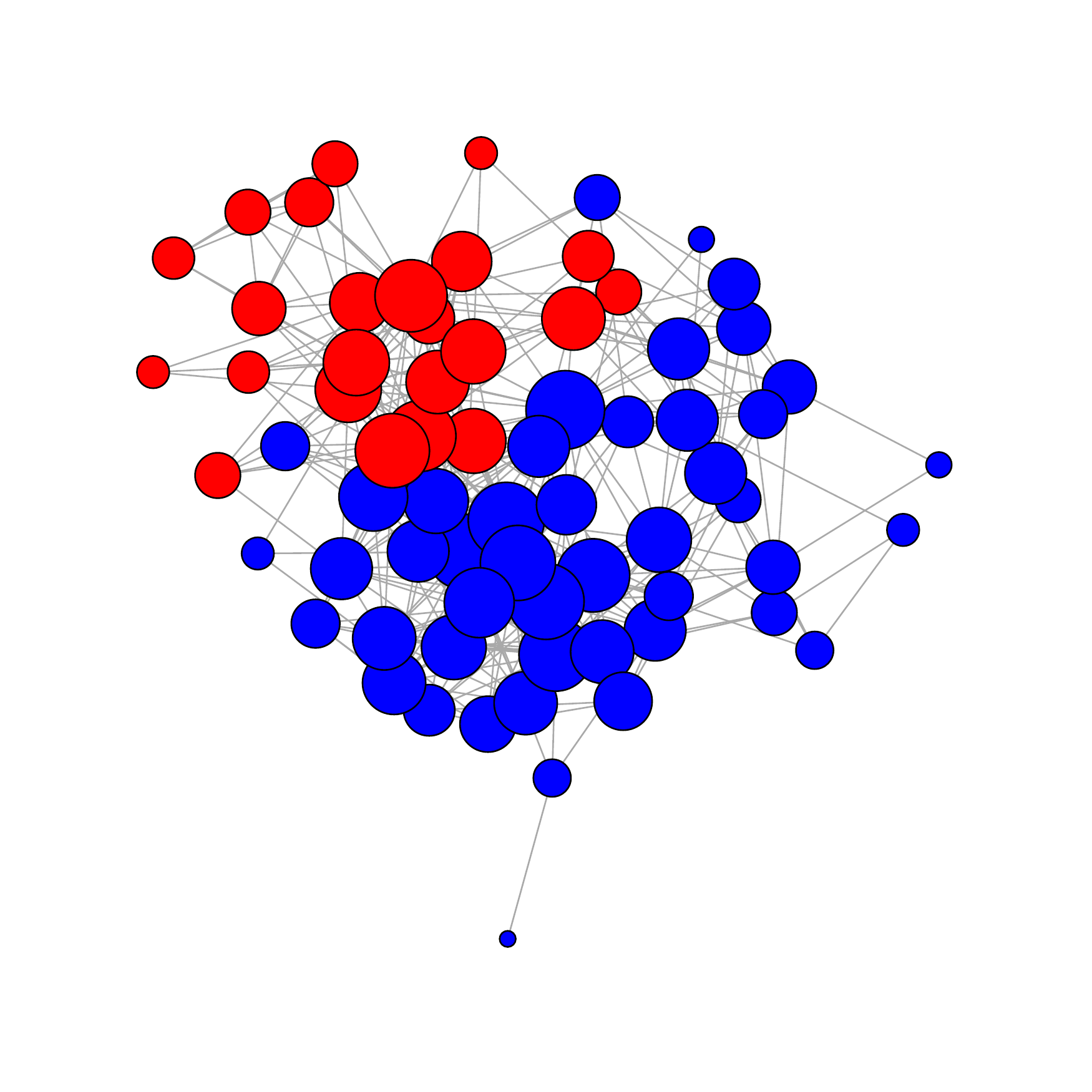}
		\caption{JCDC, $w_n=5$\\NMI=0.54}
	\end{subfigure}
	\begin{subfigure}[b]{0.18\textwidth}
		\includegraphics[width=\textwidth]{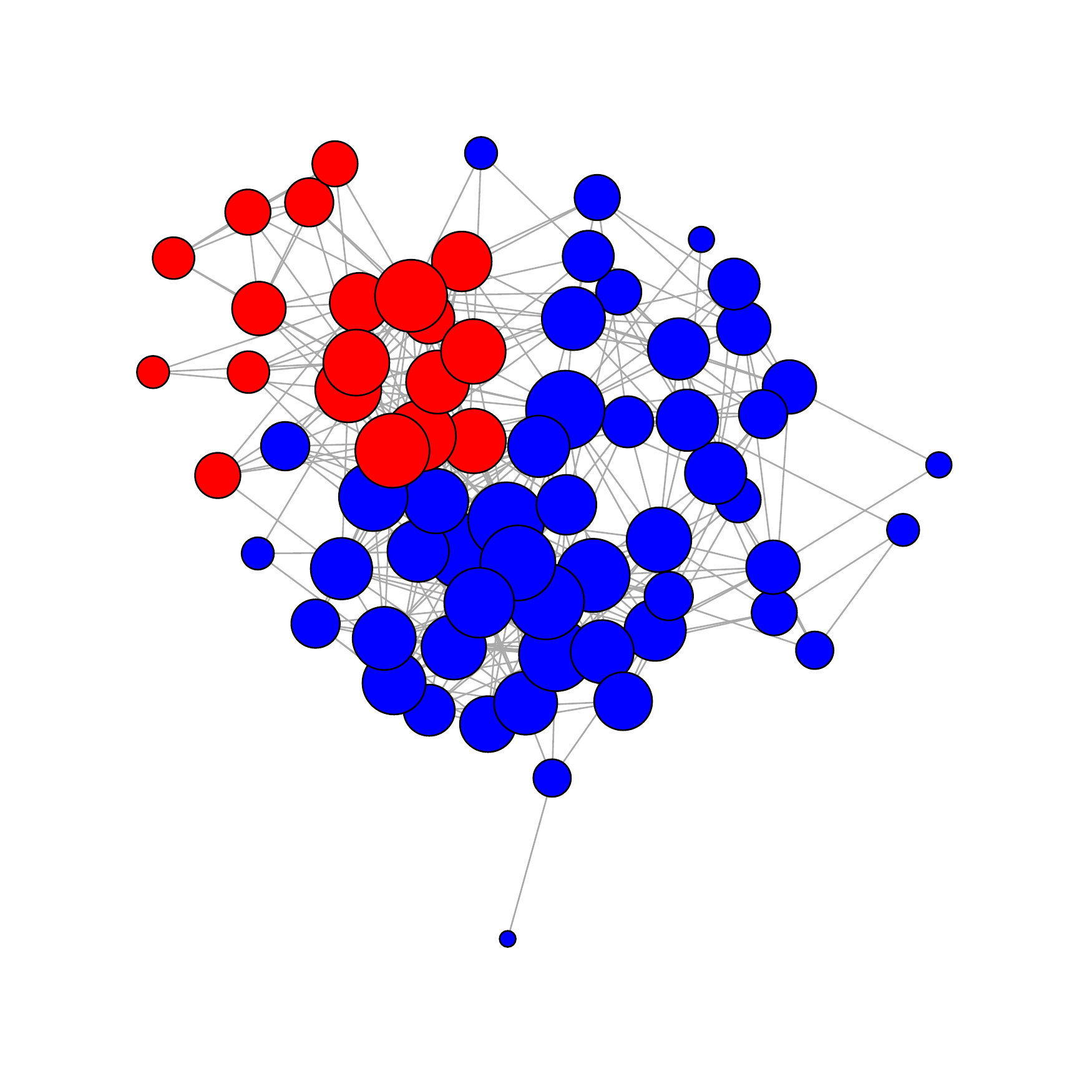}
		\caption{JCDC, $w_n=1.5$\\NMI=0.50}
	\end{subfigure}
	\begin{subfigure}[b]{0.18\textwidth}
		\includegraphics[width=\textwidth]{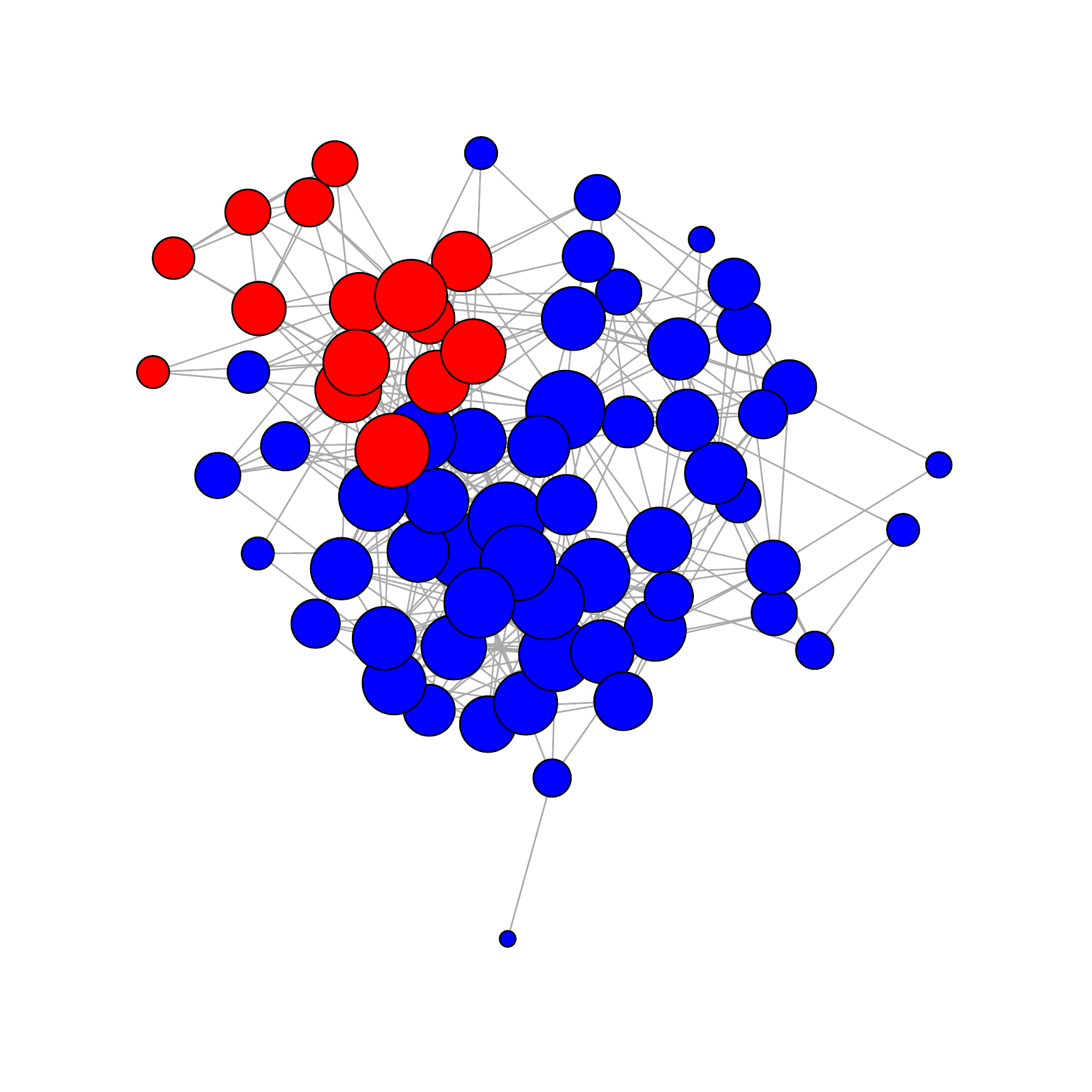}
		\caption{SC\\NMI=0.44}
	\end{subfigure}
	\begin{subfigure}[b]{0.18\textwidth}
		\includegraphics[width=\textwidth]{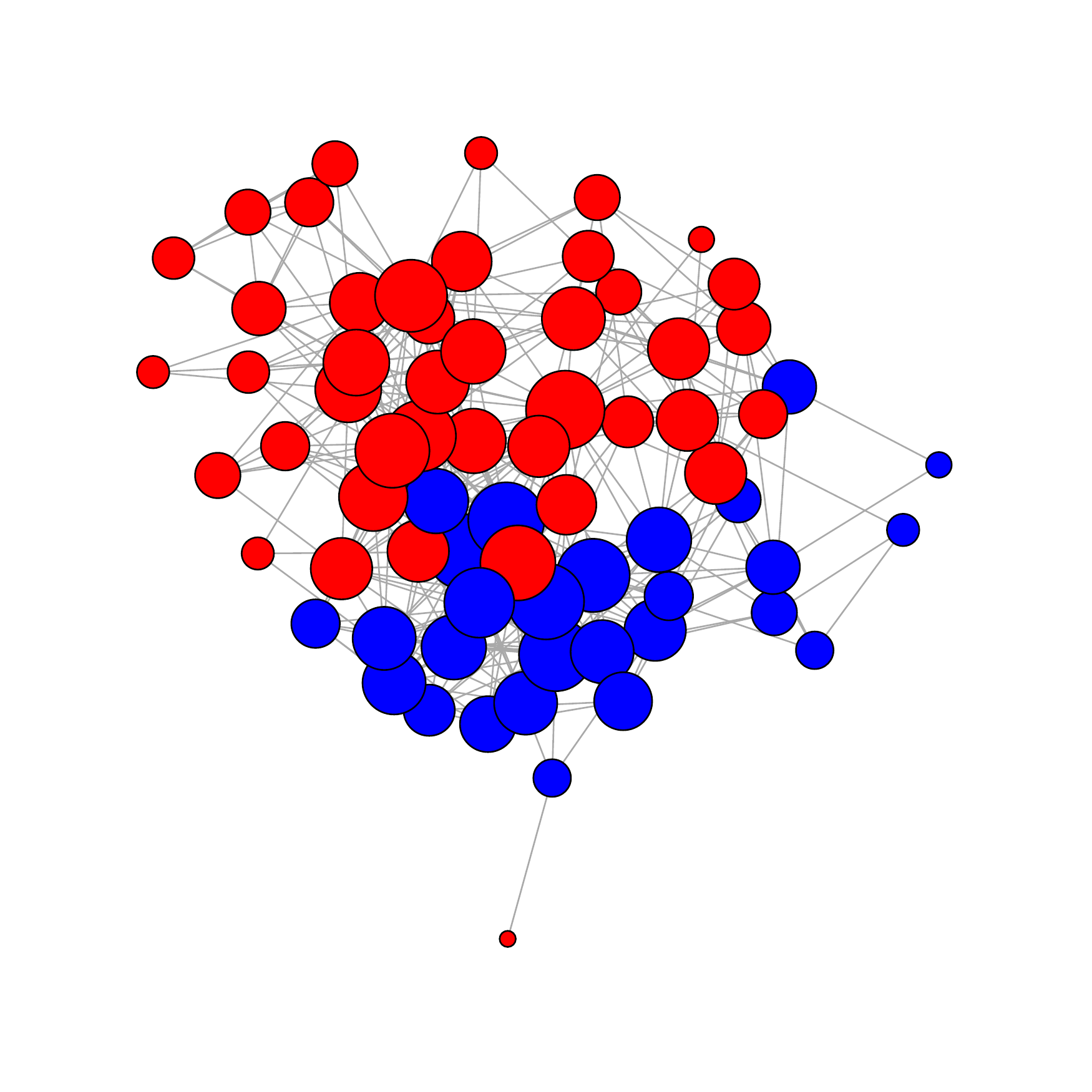}
		\caption{KM\\NMI=0.44}
	\end{subfigure}
	\begin{subfigure}[b]{0.18\textwidth}
		\includegraphics[width=\textwidth]{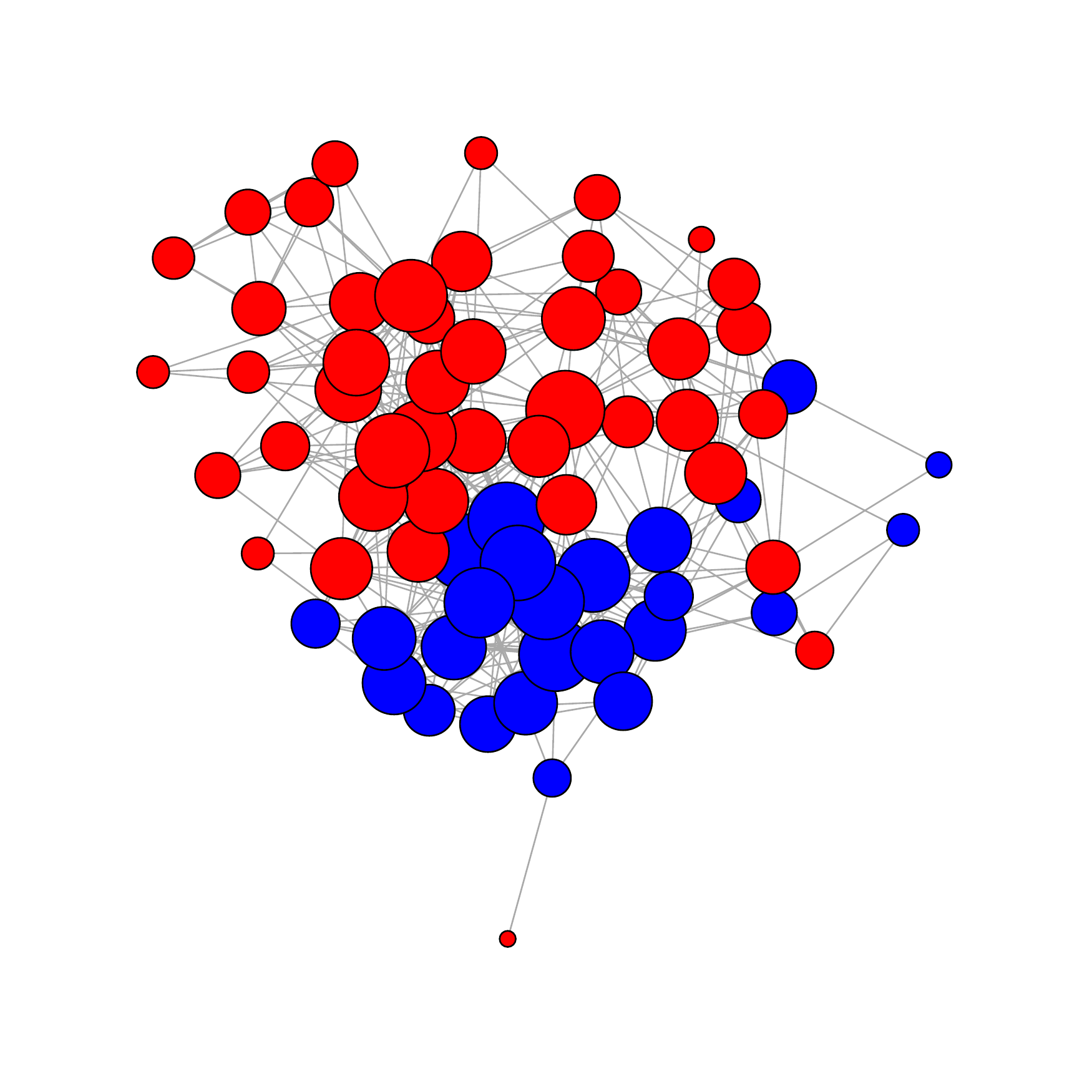}
		\caption{CASC\\NMI=0.49}
	\end{subfigure}
	\begin{subfigure}[b]{0.18\textwidth}
		\includegraphics[width=\textwidth]{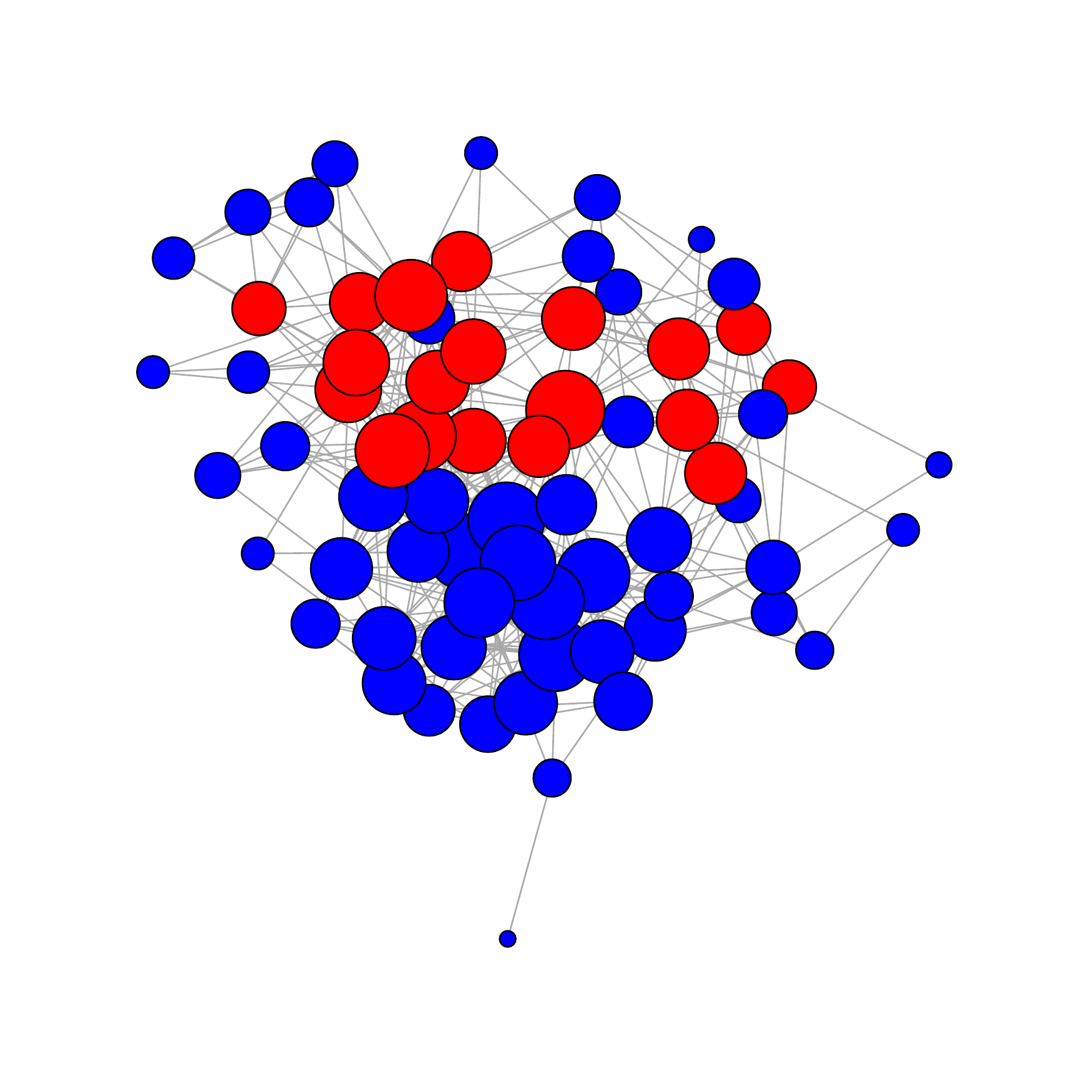}
		\caption{CESNA\\NMI=0.07}
	\end{subfigure}
	\begin{subfigure}[b]{0.18\textwidth}
		\includegraphics[width=\textwidth]{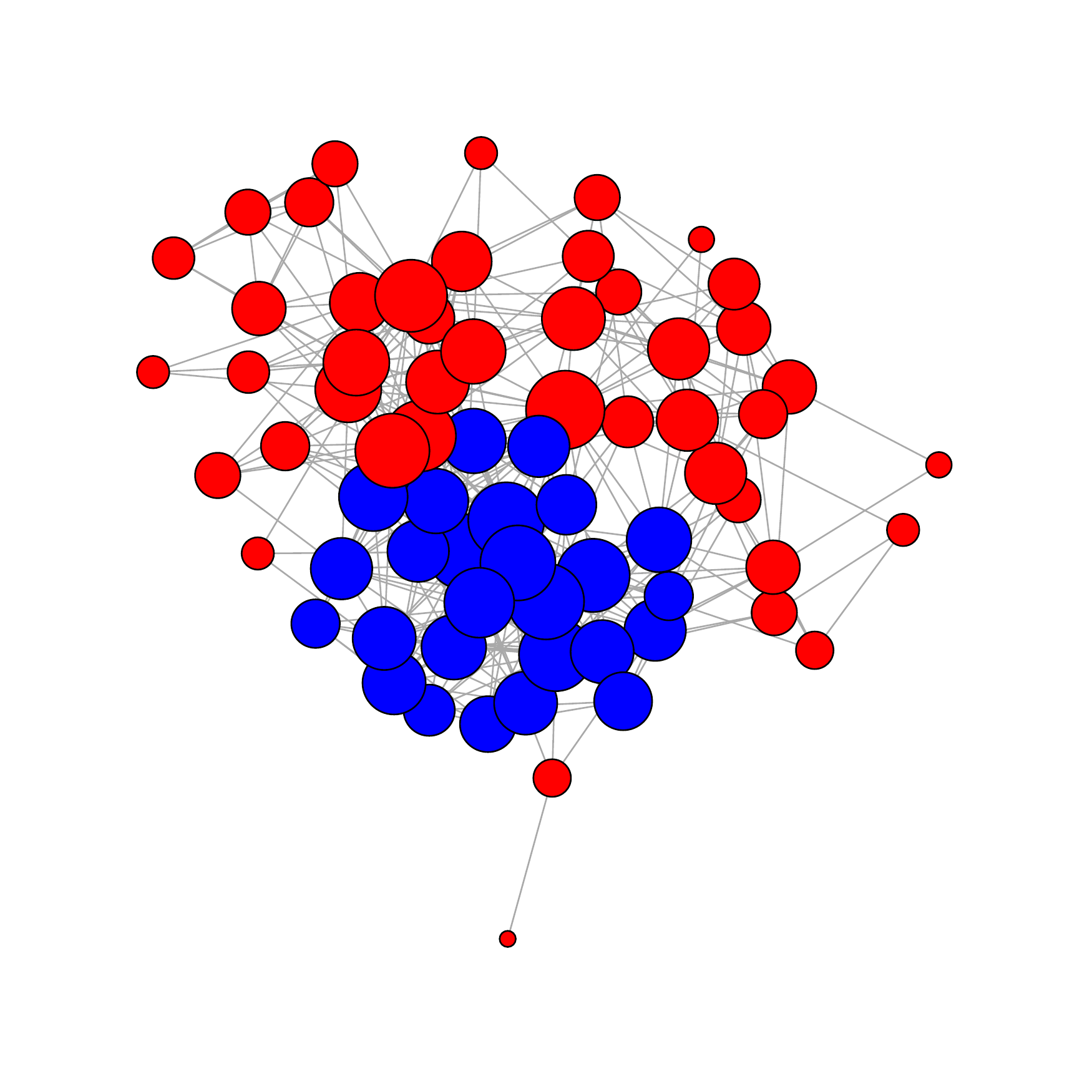}
		\caption{BAGC\\NMI=0.20}
	\end{subfigure}

	\caption{(a)-(g): adjacency matrix with nodes sorted by features; (h): network with nodes colored by status (blue is partner, red is associate); (i)-(n): community detection results from different methods. }\label{lawyers}
\end{figure}


Communities estimated by different methods are shown in Figure \ref{lawyers} (i)-(o), all run with $K = 2$. 
Spectral clustering and $K$-means have equal and reasonably high NMI values, indicating that both the adjacency matrix and node features contain community information.  JCDC obtains the highest NMI value, with $w_n=5$ performing slightly better than $w_n=1.5$.   
CASC  improves upon  spectral clustering by using the feature information, with NMI just slightly lower than that of JCDC with $w_n =1.5$. CESNA and BAGC have much lower NMI values, 
possibly because of hub nodes, or because they detect communities  corresponding to something other than status.

The estimated feature coefficients are shown in Table \ref{table:lawyers}.   Office location, years with the firm, and age appear to be the features most correlated with the community structure of status, for both partners and associates, which is natural.      Practice, school, and gender are less important, though it may be hard to estimate the influence of gender accurately since there are relatively few women in the sample.

\begin{table}[ht]
	\centering
\caption{Feature coefficients $\hat{\beta}_k$, JCDC with $w_n=5$.}
 \begin{tabular}{lllllll}
Comm.	& gender	& office& years	& age		& practice		& school\\\hline
partner		& 0.290 & 0.532 & 0.212	& 0.390 & 0.095 & 0.000\\	
associate			& 0.012 & 0.378 & 0.725	& 0.320 & 0.118 & 0.097\\\hline
\end{tabular}
\label{table:lawyers}
\end{table}


%

\section{Discussion}
Our method incorporates feature-based weights into a community detection criterion, improving detection compared to using just the adjacency matrix or the node features alone, if the cluster structure in the features is related to the community structure in the adjacency matrix.   It has the ability to estimate coefficients for each feature within each community and thus learn which features are correlated with the community structure.  This ability guards against including noise features which can mislead community detection.    The community detection criterion we use is designed for assortative community structure, with more connections within communities than between, and benefits the most from using features that have a similar clustering structure.   

This work can be extended in several directions.   Variation in node degrees, often modeled via the degree-corrected stochastic block model \cite{Karrer10} which regards degrees as independent of community structure, may in some cases be correlated with node features, and accounting for degree variation jointly with features can potentially further improve detection.  Another useful extension is to overlapping communities.  One possible way to do that is to optimize each summand in JCDC \eqref{JCDC} separately and in parallel, which can create overlaps, but would require careful initialization.  Statistical models that specify exactly how features are related to community assignments and edge probabilities can also be useful, though empirically we found no such standard models that could compete with the non-model-based JCDC on real data.   This suggests that more involved and perhaps data-specific modeling will be necessary to accurately describe real networks, and some of the techniques we proposed, such as community-specific feature coefficients, could be useful in that context.

\section*{Acknowledgments}
	E.L. is partially supported by NSF grants DMS-1106772 and DMS-1159005. J.Z. is partially supported by a NSF grant DMS-1407698 and a NIH grant R01GM096194.
	
\section*{Appendix}
\section*{A.1 \ \  Choice of tuning parameters}
\label{sec:tuning}

The JCDC method involves two user-specified tuning parameters, $\alpha$ and $w_n$.  In this section, we investigate the impact of these tuning parameters on community detection results via numerical experiments.  

First we study the impact of $\alpha$, which determines the algorithm's preference for larger or smaller communities.  We study its effect on the estimated community size as well as on the accuracy of estimated community labels.  We generate data from a stochastic block model with $n=120$ nodes and $K=2$ communities of sizes $n_1$ and $n_2 = n - n_1$.  We set the within-community edge probabilities to  $0.3$ and between-community edge probabilities to $0.15$, and vary $n_1$ from 60 to 110.   Since $\alpha$ 
is not related to feature weights, we set features to a constant, resulting in unweighted networks.  The results are averaged over 50 replications and shown in Figure \ref{fig_2}.  

\begin{figure}[h!]
\centering
\begin{subfigure}[b]{0.49\textwidth}
\centering
\includegraphics[width=\textwidth, page=1]{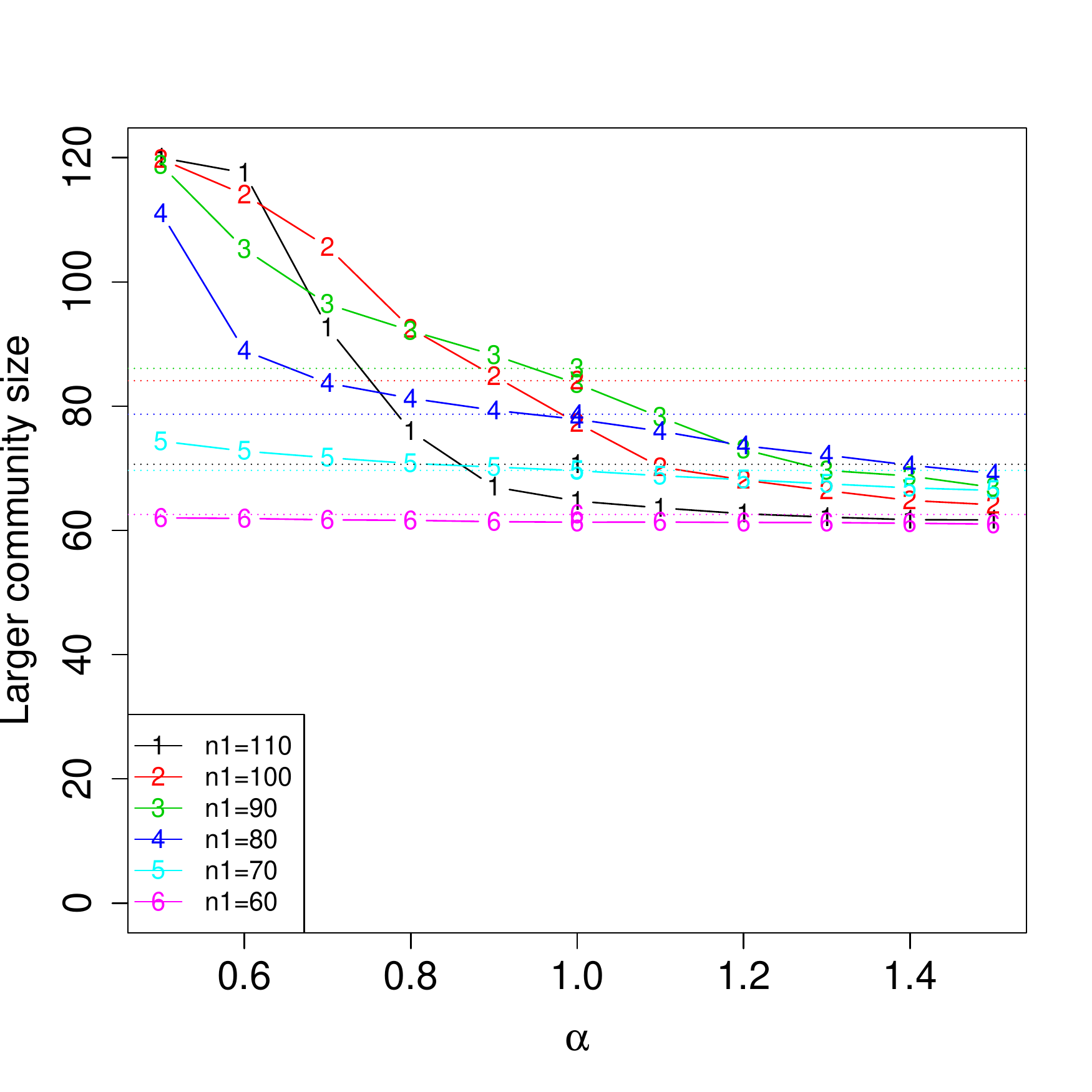}
\caption{The size of the larger estimated community}
\label{fig_2a}
\end{subfigure}
\begin{subfigure}[b]{0.49\textwidth}
\centering
\includegraphics[width=\textwidth, page=2]{fig/alpha.pdf}
\caption{Community detection accuracy} 
\label{fig_2b}
\end{subfigure}
\caption{(a) The size of the larger estimated community as a function of the tuning parameter $\alpha$.   (b)  Estimation accuracy measured by NMI as a function of the tuning parameter $\alpha$.  Solid lines correspond to JCDC and horizontal dotted lines correspond to spectral clustering, which does not depend on $\alpha$.}
\label{fig_2}
\end{figure}

We report the size of the larger estimated community in Figure \ref{fig_2}(a), and the accuracy of community detection as measured by normalized mutual information (NMI) in Figure \ref{fig_2}(b). For comparison, we also record the results from spectral clustering (horizontal lines in Figure \ref{fig_2}), which do not depend on $\alpha$. 
When communities are balanced ($n_1 = n_2 = 60$), JCDC performs well for all values of $\alpha$, producing balanced communities and uniformly outperforming spectral clustering in terms of NMI.  
In general, larger values of $\alpha$ in JCDC result in more balanced communities, while smaller $\alpha$'s tend to produce a large and a small community.
In terms of community detection accuracy, Figure \ref{fig_2}(b) shows that the JCDC method outperforms spectral clustering over a range of values of $\alpha$, and this range depends on how unbalanced the communities are.  For simplicity and ease of interpretation, we set $\alpha = 1$ for all the simulations and data analysis reported in the main manuscript;  however, it can be changed by the user if information about community sizes is available.  


Next, we investigate the impact of $w_n$, which controls the influence of features.  To study the trade-off between the two sources of information (network and features),  we generate two different community partitions.  Specifically, we consider two communities of sizes $n_1$ and $n_2$, with $n_1 + n_2 = n  = 120$. We generate two label vectors $c^A$ and $c^F$, with $c^A_i=1$ for $i=1,\ldots, n_1$ and $c^A_i=2$ for $i=n_1+1,\ldots, n$, while the other label vector has $c^F_i=1$ for $i=1,\ldots, n_2$ and $c^F_i=2$ for $i=n_2+1, \ldots, n$.  Then the edges are generated from the stochastic block model based on $c^A$, and the node features are generated based on $c^F$.  We generate two node features: one feature is sampled from the distribution $N(\mu, 1)$  if $c_i^F=1$ and $N(0, 1)$ if $c_i^F=2$; the other feature is sampled from $N(0, 1)$ if $c_i^F=1$ and $N(-\mu, 1)$ if $c_i^F=2$. We fix $\mu=3$ and set $\alpha=1$, as discussed above. We set the within- and between-community edge probabilities to 0.3 and 0.15, respectively, same as in the previous simulation, and vary the value of $w_n$ from $1.1$ to $10$.  Finally, we look at the the agreement between the estimated communities $\hat{e}$ and $c_A$ and $c_F$, as measured by normalized mutual information.  The results are shown in Figure \ref{fig_3}.  

\begin{figure}[h!]
\centering
\includegraphics[width=0.5\textwidth]{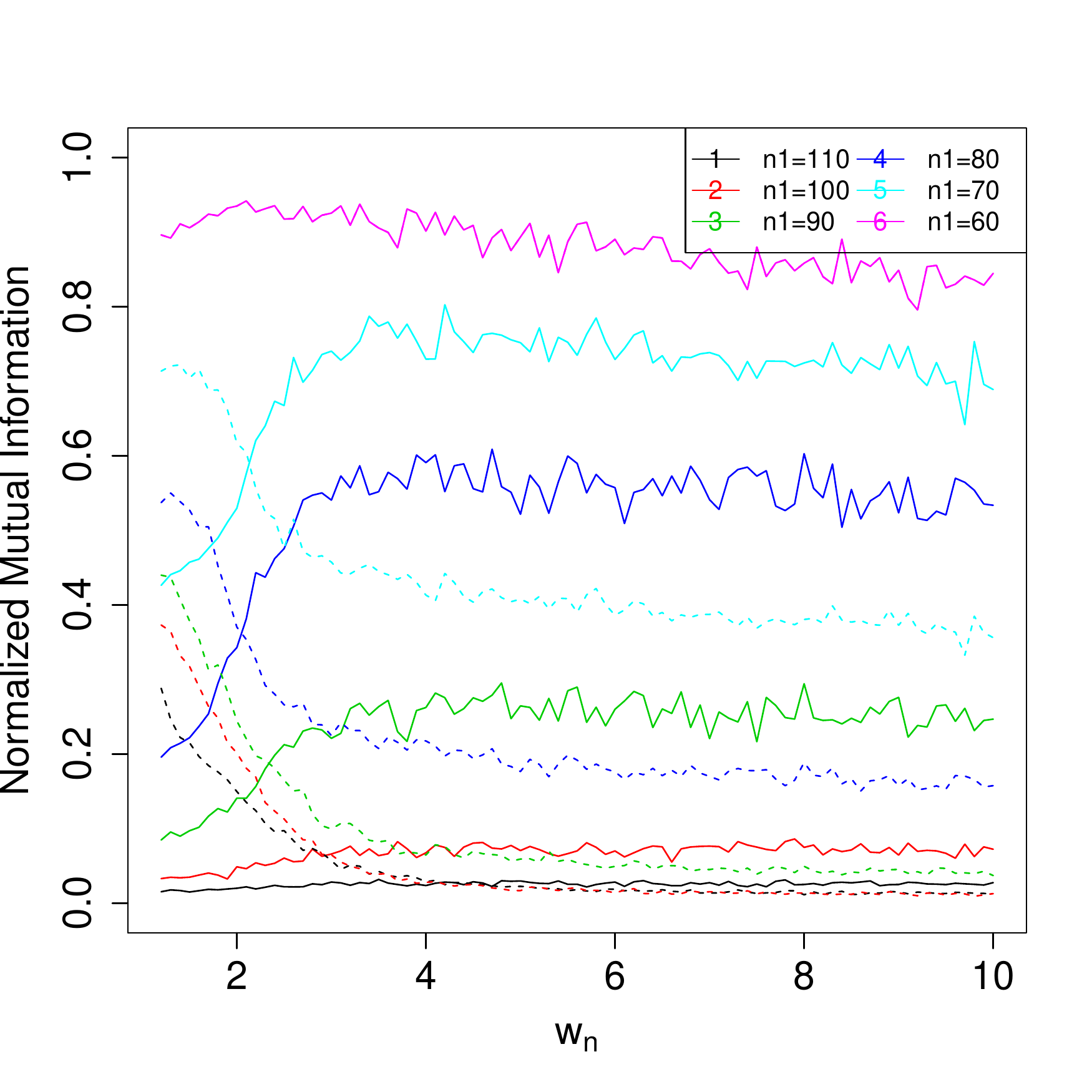}
\caption{MNI between the estimated community structure $\hat{e}$ and the network community structure $c_A$ (solid lines) and the feature community structure $c_F$ (dotted lines).   Note that when $n_1=n_2=60$, $c^A=c^F$, so the solid and dotted lines coincide.}
\label{fig_3}
\end{figure}

As we expect,  smaller values of $w_n$ give more influence to features and thus the estimated community structure agrees better with $c^F$ than with $c^A$.  As $w_n$ increases, the estimated $\hat{e}$ becomes closer to $c^A$.    In the manuscript, we compare two values of $w_n$, 1.5 and 5.


\section*{A.2 \ \  Proofs}
We start with summarizing notation. Let $\E_1, \dots, \E_K$ be the  estimated communities corresponding to the label vector $e$, and ${\cal C}_1, \ldots, {\cal C}_K$  the true communities corresponding to the label vector $c$.   
Recall we estimate $e$ by maximizing the criterion $R$ over $e$ and $\beta$, where
\begin{equation*}
	R(e, \beta; w_n) = \sum_{k=1}^K\frac{1}{|\E_k|^\alpha}\sum_{i, j\in \E_k} A_{ij}W(\phi_{ij}, \beta_k; w_n)  \ , 
\end{equation*}
and define
\begin{equation*}
	\hat{e} = \arg\max_e\left( \max_\beta R(e, \beta; w_n) \right) \ , 
\end{equation*}
where $\hat{e}$ and the corresponding $\hat \beta$ are defined up to a permutation of community labels.   Recall that we assumed $A$ and $F$ are conditionally independent given $c$ and defined $R^0$, the ``population version'' of $R$,  as 
\begin{equation*}
	R^0(e, \beta; w_n) = \sum_{k=1}^K\frac{1}{|\E_k|^\alpha}\sum_{i, j\in \E_k} \rho_n P_{c_ic_j}\mathbb{E}[W(\phi_{ij}, \beta_k; w_n)] \ .
\end{equation*}
The expectation in $R^0$ is taken with respect to the distribution of node features, which determine the similarities $\phi_{ij}$.  

\begin{lemma}\label{lemma_1}
Under conditions \ref{condition_1} and \ref{condition_2}, if $w_n\rho_n\to \infty$ and $0<\alpha\leq 2$, we have
\begin{align*}
  \max_{e,\beta} \frac{\left| R(e, \beta; w_n) - R^0(e, \beta; w_n) \right|}{w_n \rho_n n^{2-\alpha}} = O_p\left(\frac{1}{\sqrt{w_n\rho_n}}\right) \ .
\end{align*}
\end{lemma}

\begin{proof}[Proof of Lemma~\ref{lemma_1}]
We first bound the difference between $R$ and $R^0$ for fixed $e$ and $\beta$. By Hoeffding's inequality and the fact that $2[n/2] \ge n-1$, where $[x]$ is the integer part of $x$, we have 
\begin{align*}
  \prb\left\{\left| \frac{1}{|\E_k|^2} \sum_{i, j\in \E_k} \left( A_{ij}\wijbkwn-\rho_nP_{c_ic_j} \ep[\wijbkwn] \right) \right| > t \right\}\leq 2\exp\left(  - (|\E_k|-1) t^2 \right) \ .
\end{align*}
Taking $t=w_n\rho_n n^{2-\alpha} |\E_k|^{\alpha- 2} \delta$ and applying the union bound, we have
\begin{align*}
\prb  & \left( \frac{|R(e, \beta; w_n) - R^0(e, \beta; w_n)|}{w_n\rho_n n^{2-\alpha}}  > K\delta \right) \\
\leq & \sum_{k=1}^K \prb \left\{ \frac{ \big| \sum_{i,j\in \E_k} \left( A_{ij}\wijbkwn - \rho_nP_{c_ic_j} \ep\left[ \wijbkwn \right] \right) \big| }{ w_n\rho_n|\E_k|^\alpha n^{2-\alpha} } \geq \delta \right\}\\
\leq & \sum_{k=1}^K 2\exp\left\{ - (|\E_k|-1) w_n^2 \rho_n^2 n^{4-2\alpha} |\E_k|^{2\alpha-4} \delta^2 \right\} \leq 2K\exp\left\{ -(\pi_0 n-1)w_n^2\rho_n^2\delta^2 \right\} \ .
\end{align*}
Next, we take the uniform bound over $\beta$. Consider the set
\begin{equation*}
B_\epsilon =\left\{ \left(\frac{s_1 \epsilon}{ \sqrt{p}},\ldots, \frac{s_p \epsilon}{ \sqrt{p}} \right), s_1,\ldots,s_p\in\left\{ 0, \pm 1,\ldots,\pm \left[ \frac{M_\beta\sqrt{p}}{\epsilon} \right] , \pm\frac{M_\beta\sqrt{p}}{\epsilon} \right\} \right\} \ .
\end{equation*}
It is straightforward to verify that $B_\epsilon$ is an $\epsilon$-net on $[-M_\beta, M_\beta]^p$, the space of $\beta_k$'s. For each $\beta_k$, let $\beta(\beta_k, B_\epsilon)$ be the best approximation to $\beta_k$ in $B_\epsilon$.  Then 
\begin{align*}
\max_{\beta_k} |\wijbkwn - W(\phi_{ij}, \beta(\beta_k, B_\epsilon); w_n)|&\leq \max_{\beta_k} \left| \frac{\partial W}{\partial \beta_k}(\phi_{ij}, \beta_k; w_n) \right| |\beta_k - \beta(\beta_k, B_\epsilon)|\\
&\leq 2M_\phi M_\beta \exp(M_\phi M_\beta) \epsilon \leq 2M_\phi M_\beta w_n\epsilon
\end{align*}
Therefore, choosing $\epsilon  = \frac{\rho_n \delta}{4M_\phi M_\beta}$, we have
\begin{align*}
& \prb\left( \max_\beta \frac{|R(e, \beta; w_n) - R^0(e, \beta; w_n)|}{w_n\rho_n n^{2-\alpha}} > K\delta \right)\leq \sum_{k=1}^K \prb\left\{ \max_{\beta_k} \frac{\left| \sum_{i,j \in {\cal E}_k} (A_{ij} - \rho_nP_{c_ic_j})\wijbkwn \right|}{w_n\rho_n|\E_k|^\alpha n^{2-\alpha}} >\delta \right\}\\
\leq & \sum_{k=1}^K \prb\left\{ \max_{\beta_k} \frac{ \sum_{i,j\in\E_k} |A_{ij} - \rho_nP_{c_ic_j}| |\wijbkwn - W(\phi_{ij}, \beta(\beta_k, B_\epsilon); w_n)| }{w_n\rho_n|\E_k|^\alpha n^{2-\alpha}} > \frac{\delta}{2} \right\}\\
&  +  \sum_{k=1}^K \prb\left\{ \max_{\beta_0\in B_\epsilon} \frac{ \left| \sum_{i,j\in\E_k} \left(A_{ij} - \rho_nP_{c_ic_j}\right) W(\phi_{ij}, \beta_0; w_n) \right| }{w_n\rho_n|\E_k|^\alpha n^{2-\alpha}} > \frac{\delta}{2} \right\}\\
\leq & K \prb\left( \frac{|\E_k|^{2-\alpha}\cdot 2M_\phi M_\beta \epsilon}{\rho_n n^{2-\alpha}} \geq \frac{\delta}{2} \right) + 2K |B_\epsilon| \exp\left\{ -(\pi_0 n-1)w_n^2\rho_n^2\delta^2/4 \right\}\\
\leq & 0 + 2K\left( \frac{4M_\phi M_\beta^2\sqrt{p}}{\rho_n\delta} + 3 \right)^p\exp\left\{ -(\pi_0 n-1)w_n^2\rho_n^2\delta^2/4 \right\} \ , 
\end{align*}
where the first term becomes 0 because of the choice of $\epsilon$ and $|\E_k| < n$.  Finally, taking a union bound over all possible community assignments, we have
\begin{align*}
\prb\left( \max_{e,\beta}\frac{|R(e, \beta; w_n) - R^0(e, \beta; w_n)|}{
w_n\rho_n n^{2-\alpha}} > K\delta \right) & \leq 2K^{n+1}\left( \frac{4M_\phi M_\beta^2\sqrt{p}}{\rho_n\delta} + 3 \right)^p\exp\left\{ -(\pi_0 n-1)w_n^2\rho_n^2\delta^2/4 \right\}\\
&\leq 2K \exp\left[ -\pi_0 nw_n^2\rho_n^2\delta^2/8 + n\log K + p\log\{ C_1/(\rho_n \delta) \} \right] \ ,
\end{align*}
where $C_1:=4M_\phi M_\beta^2\sqrt{p}$.  Taking $\delta = 1/\sqrt{w_n\rho_n}$ completes the proof of Lemma \ref{lemma_1}.
\end{proof}

We now proceed to investigate the ``population version''  of our criterion,  $R^0$. Define $U \in\mathbb{R}^{K\times K}$ by $U_{kl} = \sum_{i=1}^n 1[e_i=k, c_i=l]/n$, and let $D$ be a diagonal $K \times K$ matrix with $\pi_1, \dots, \pi_K$ on the diagonal, where $\pi_k =\sum_{i=1}^n 1[c_i=k]/n$ is the fraction of nodes in community ${\cal C}_k$. Roughly speaking, $U$ is the confusion matrix between $e$ and $c$, and  $U=DO$ for a permutation matrix $O$ means the estimation is perfect. Define  
\begin{equation*}
  g(U) =\sum_{k=1}^K\frac{\sum_{l=1}^K \sum_{l'=1}^K U_{kl}U_{kl'}P_{ll'}}{ \left(\sum_{a=1}^K U_{ka}\right)^\alpha } \ .
\end{equation*}
Each estimated community assignment $e$ induces a unique $U = U(e)$. It is not difficult to verify that
\begin{equation*}
  g\left(U(e)\right) = \sum_{k=1}^K\frac{\sum_{i,j\in\E_k} P_{c_ic_j}}{|\E_k|^\alpha n^{2-\alpha}} \ .
\end{equation*}

\begin{lemma}\label{lemma_2}
Under conditions 1 and 2, there exists a constant $C_2$ such that
\begin{equation*}
  \max_{e,\beta} \left| \frac{R^0(e, \beta; w_n)}{w_n\rho_nn^{2-\alpha}} - g\left(U(e)\right) \right| \leq \frac{C_2}{w_n} \ . 
\end{equation*}
\end{lemma}
\begin{proof}[Proof of Lemma~\ref{lemma_2}]
By definition, we have
\begin{align*}
& \max_{e, \beta}\left| \frac{R^0(e, \beta; w_n)}{w_n\rho_nn^{2-\alpha}} - g\left(U(e)\right) \right| = \max_e\sum_{k=1}^K \max_{\beta_k} \sum_{i,j\in\E_k} P_{c_ic_j} \frac{\ep[\exp(-\langle \phi_{ij}, \beta_k \rangle)]}{|\E_k|^\alpha w_n n^{2-\alpha}}\\
\leq & \max_e\sum_{k=1}^K \sum_{i,j\in\E_k} \frac{\exp(M_\phi M_\beta)}{|\E_k|^\alpha w_n n^{2-\alpha}}\max_{kl}P_{kl} \leq \frac{K \exp(M_\phi M_\beta)}{w_n\pi_0^{2-\alpha}}\max_{kl}P_{kl} = \frac{C_2}{w_n} \ ,
\end{align*}
where $C_2 :=K  \pi_0^{\alpha-2} \exp(M_\phi M_\beta)\max_{kl}P_{kl} $, and the two inequalities follow from conditions \ref{condition_1} and \ref{condition_2}, respectively.
\end{proof}

\begin{lemma}\label{lemma_3}
Under condition \ref{condition_3}, if $\alpha\in \left[\max_{1\leq k<l\leq K} 2(K-1)P_{kl}/\min(P_{kk}, P_{ll}), 1\right]$, then for all $U$ satisfying $\sum_{k=1}^K U_{kl} = \pi_l$ for $1\leq k\leq K$, $g(U)$ is uniquely maximized at $U=DO$ for $O\in{\cal O}_K$, where ${\cal O}_K$ denotes the set of $K \times K$ permutation matrices.  
\end{lemma}

\begin{proof}[Proof of Lemma~\ref{lemma_3}]
We have
\begin{align}
g(D) - g(U) & = \sum_{l=1}^K \left(\sum_{k=1}^K U_{kl}\right)^{2-\alpha} P_{ll} - \sum_{k=1}^K\frac{ \sum_{l=1}^K U_{kl}^2 P_{ll} + \sum_{l=1}^K\sum_{l'\neq l}U_{kl}U_{kl'}P_{ll'} }{\left(\sum_{a=1}^K U_{ka}\right)^\alpha}\nonumber\\
=&\sum_{l=1}^K\left\{ \left( \sum_{k=1}^K U_{kl} \right)^{2-\alpha} - \sum_{k=1}^K \frac{U_{kl}^2}{\left(\sum_{a=1}^K U_{ka}\right)^\alpha} \right\} P_{ll} - \sum_{k=1}^K\sum_{l=1}^K\sum_{l'\neq l}\left\{ \frac{U_{kl}U_{kl'}}{\left(\sum_{a=1}^K U_{ka}\right)^\alpha} \right\} P_{ll'} \label{diff_gD_gU}
\end{align}
For $0<\alpha\leq 1$, since $U_{kl} \ge 0$  for all $k$ and $l$, we have $\left( \sum_{k=1}^K U_{kl} \right)^{2-\alpha} \geq \sum_{k=1}^K U_{kl}^{2-\alpha}$.  By mid-value theorem, there exists $\xi_{kl}\in\left(0, \sum_{a\neq l}U_{ka}\right)$, such that 
\begin{equation}
\left( \sum_{a=1}^K U_{ka} \right)^\alpha - U_{kl}^\alpha = \alpha\left( \sum_{a\neq l}U_{ka} \right) / \left(U_{kl} + \xi^{kl}\right)^{1-\alpha} \geq \alpha\left( \sum_{a\neq l}U_{ka} \right) / \left( \sum_{a=1}^K U_{ka} \right)^{1-\alpha}.  \label{mid_value}
\end{equation}
Finally, we will need the following inequality: for $0 < \alpha \le 2$ and $x, y\geq 0$ satisfying $x+y\leq u$,
\begin{equation}
x^{2-\alpha}(u-x) + y^{2-\alpha}(u-y)\geq x y u^{1-\alpha} \ . \label{claim_xyu}
\end{equation}
For $x= y = 0$, equality holds.   To verify \eqref{claim_xyu} when $0 < x+y \le u$, dividing by $u^{3-\alpha}$ we have
\begin{align*}
&\frac{x^{2-\alpha}(u-x) + y^{2-\alpha}(u-y) - xyu^{1-\alpha}}{u^{3-\alpha}} = \left(\frac{x}{u}\right)^{2-\alpha} \left(1-\frac{x}{u}\right) + \left(\frac{y}{u}\right)^{2-\alpha} \left(1-\frac{y}{u}\right) - \frac{xy}{u^2}\\
\geq & \left(\frac{x}{u}\right)^2 \left(1-\frac{x}{u}\right) + \left(\frac{y}{u}\right)^2 \left(1-\frac{y}{u}\right) - \frac{xy}{u^2} = \left\{ \left(\frac{x}{u}\right)^2 + \left(\frac{y}{u}\right)^2 - \frac{xy}{u^2} \right\}\left(1-\frac{x+y}{u}\right)\geq 0 \ .
\end{align*}
The first inequality above implies that a necessary condition for equality to hold in \eqref{claim_xyu} is $xy=0$.

We now lower bound the first term on the right hand side of \eqref{diff_gD_gU}.
\begin{align}
&\sum_{l=1}^K\left\{ \left( \sum_{k=1}^K U_{kl} \right)^{2-\alpha} - \sum_{k=1}^K \frac{U_{kl}^2}{\left(\sum_{a=1}^K U_{ka}\right)^\alpha} \right\} P_{ll} \geq \sum_{l=1}^K\sum_{k=1}^K\frac{ U_{kl}^{2-\alpha}\left\{ \left(\sum_{a=1}^K U_{ka}\right)^\alpha - U_{kl}^\alpha \right\} }{\left(\sum_{a=1}^K U_{ka}\right)^\alpha}P_{ll}\nonumber\\
\geq & \sum_{l=1}^K\sum_{k=1}^K \frac{U_{kl}^{2-\alpha}\left(\sum_{a\neq l}U_{ka}\right)}{\sum_{a=1}^K U_{ka}} \alpha P_{ll}\geq \sum_{l=1}^K\sum_{k=1}^K \frac{U_{kl}^{2-\alpha}\left(\sum_{a\neq l}U_{ka}\right)}{\sum_{a=1}^K U_{ka}} \sum_{l'\neq l} 2P_{ll'}\nonumber\\
=& \sum_{k=1}^K\left\{ \sum_{l=1}^K\sum_{l'\neq l} \frac{U_{kl}^{2-\alpha}\left(\sum_{a\neq l} U_{ka}\right) P_{ll'}}{\sum_{a=1}^K U_{ka}}  +  \sum_{l'=1}^K\sum_{l\neq l'} \frac{U_{kl'}^{2-\alpha}\left(\sum_{a\neq l'} U_{ka}\right) P_{ll'}}{\sum_{a=1}^K U_{ka}}  \right\}\nonumber\\
=& \sum_{k=1}^K \sum_{l=1}^K\sum_{l'\neq l} \frac{U_{kl}^{2-\alpha}\left(\sum_{a\neq l}U_{ka}\right) + U_{kl'}^{2-\alpha}\left(\sum_{a\neq l'} U_{ka}\right)}{\sum_{a=1}^K U_{ka}} P_{ll'}      \geq \sum_{k=1}^K\sum_{l=1}^K\sum_{l'\neq l} \frac{U_{kl}U_{kl'}}{\left(\sum_{a=1}^K U_{ka}\right)^\alpha}P_{ll'}  \ ,  \label{bound_first_term_diff_gD_gU}
\end{align}
where the last equality is obtained by applying \eqref{claim_xyu} with $x=U_{kl}$, $y=U_{kl'}$ and $u = \sum_{a=1}^K U_{ka}$. Plugging \eqref{bound_first_term_diff_gD_gU} into \eqref{diff_gD_gU}, we have
\begin{equation*}
g(D) - g(U) \geq 0 \ .
\end{equation*}

It remains to show that equality holds only if $U = DO$ for some $O \in {\cal O}_K$.   Note that the last inequality in \eqref{bound_first_term_diff_gD_gU} is obtained from \eqref{claim_xyu}, where equality holds only when $xy = 0$. The corresponding condition for equality to hold in \eqref{bound_first_term_diff_gD_gU} is thus $U_{kl} U_{kl'}=0$ for all $k$, $l$ and $l'$. Therefore, for each $k$, there is only one $l$ such that $U_{kl}\neq 0$, 
 i.e., $U = DO$ for some $O \in {\cal O}_K$.
 \end{proof}

\begin{proof}[Proof of Theorem 1]
By Lemma \ref{lemma_1} and Lemma \ref{lemma_2}, we have
\begin{equation}
\max_{e,\beta}\left|\left\{\frac{R(e, \beta; w_n)}{w_n\rho_n n^{2-\alpha}} - g\left(U(e)\right)\right\} \right| = O_p\left(\frac{1}{\sqrt{w_n\rho_n}}\right)    \label{diff_R_gU} \ .  
\end{equation}
It is straightforward to verify that, for any $e$, $2d(e, c)=\min_{O\in\mathcal{O}_K}\|U(e)-DO\|_1$,  where $\|Q\|_1 =\sum_{k=1}^K\sum_{l=1}^K |Q_{kl}|$.  Take a sequence of decreasing positive numbers $x_n\to 0$ and define
\begin{equation}
y_n =\max_{U: g(D) - g(U)\leq x_n} \min_{O\in {\cal O}_K}\|U-DO\|_1
\end{equation}
We now show, by contradiction, that  $x_n\to0$ implies $y_n\to0$.   First, note that $y_n$ is non-increasing. Now if $y_0 =\lim_{n\to\infty}y_n>0$, by compactness of the set $\mathcal{U}_{y_0} = \{U: \min_{O\in {\cal O}_K }\|U-DO\|_1\geq y_0\}$ and continuity of the function $g$, the supremum of $g(U)$ over $U\in\mathcal{U}_{y_0}$, which equals $g(D)$, is attained in $\mathcal{U}_{y_0}$. This contradicts Lemma \ref{lemma_3}.

Now let $x_n = 1/\sqrt[4]{w_n\rho_n}$. By assumption of Theorem 1, $x_n\to0$, which yields $y_n\to0$. Also $x_n/\left(1/\sqrt{w_n\rho_n}\right) = \sqrt[4]{w_n\rho_n} \to\infty$, so by \eqref{diff_R_gU} we have
\begin{align}
\prb\left[  \left\{ \left|\frac{R(\hat{e}, \hat{\beta}; w_n)}{w_n\rho_n n^{2-\alpha}} - g\left(U(\hat{e})\right) \right| > \frac{x_n}{2} \right\} \bigcup \left\{ \left|\frac{R(c, \beta; w_n)}{w_n\rho_n n^{2-\alpha}} - g\left(D\right) \right| > \frac{x_n}{2} \right\} \right]\to 0   \ .    \label{sample_approach_pop}
\end{align}
Now, the event
$$
\left|\frac{R(\hat{e}, \hat{\beta}; w_n)}{w_n\rho_n n^{2-\alpha}} -  g\left(U(\hat{e})\right) \right| \leq \frac{x_n}{2}  \textrm{  and  }  \left|\frac{R(c, \beta; w_n)}{w_n\rho_n n^{2-\alpha}} - g\left(D\right)\right| \leq \frac{x_n}{2}
$$
implies that $g(D)-g(U(\hat{e}))\leq   \frac{R(c, \beta; w_n)}{w_n\rho_n n^{2-\alpha}} - \frac{R(\hat{e}, \hat{\beta}; w_n)}{w_n\rho_n n^{2-\alpha}} + x_n  \leq x_n$. So we have
\begin{equation}
	\prb\left( g(D) - g(U(\hat{e})) \leq x_n \right) \to 1   \label{convergence_2}
\end{equation}
and 
\begin{align*}
	2d(\hat{e}, c) &= \min_{O\in{\cal O}_K}\|U(\hat{e}) - DO\|_1 \leq \max_{U:g(D)-g(U)\leq x_n} \min_{O\in{\cal O}_K}\|U - DO\|_1 = y_n \to 0 .  
\end{align*}

\end{proof}

\bibliography{Manuscript_1}

\begin{thebibliography}{10}

\bibitem{korean1}
Rogers EM, Kincaid DL
\newblock (1981) Communication networks: Toward a new paradigm for research.
\newblock (Free Press NY)

\bibitem{schlitt2007current}
Schlitt T, Brazma A
\newblock (2007) Current approaches to gene regulatory network modelling.
\newblock \emph{BMC bioinformatics} 8:S9.

\bibitem{Holland83}
Holland PW, Laskey KB, Leinhardt S
\newblock (1983) Stochastic blockmodels: first steps.
\newblock \emph{Social Networks} 5:109--137.

\bibitem{Karrer10}
Karrer B, Newman MEJ
\newblock (2011) Stochastic blockmodels and community structure in networks.
\newblock \emph{Phys. Rev. E} 83:016107.

\bibitem{Hoff2007}
Hoff PD
\newblock (2008) Modeling homophily and stochastic equivalence in symmetric relational data.
\newblock \emph{Adv. NIPS} pp 657--664.

\bibitem{shi00}
Shi J, Malik J
\newblock (2000) Normalized cuts and image segmentation.
\newblock \emph{\mbox{IEEE} Trans. PAMI} 22:888--905.

\bibitem{Newman&Girvan2004}
Newman MEJ, Girvan M
\newblock (2004) Finding and evaluating community structure in networks.
\newblock \emph{Phys. Rev. E} 69:026113.

\bibitem{NewmanPNAS}
Newman MEJ
\newblock (2006) Modularity and community structure in networks.
\newblock \emph{Proc. Natl. Acad. Sci. USA} 103:8577--8582.

\bibitem{qin2013regularized}
Qin T, Rohe K
\newblock (2013) Regularized Spectral Clustering under the Degree-Corrected Stochastic Blockmodel.
\newblock \emph{Adv. NIPS} pp 3120--3128.

\bibitem{steglich2006applying}
Steglich C, Snijders TAB, West P
\newblock (2006) Applying siena.
\newblock \emph{Methodology} 2:48--56.

\bibitem{snijders2006new}
Snijders TA, Pattison PE, Robins GL, Handcock MS
\newblock (2006) New specifications for exponential random graph models.
\newblock \emph{Socio. Method.} 36:99--153.

\bibitem{hummon1990analyzing}
Hummon NP, Doreian P, Freeman LC
\newblock (1990) Analyzing the structure of the centrality-productivity literature created between 1948 and 1979.
\newblock \emph{Sci. Comm.} 11:459--480.

\bibitem{Jure_ego}
McAuley J, Leskovec J
\newblock (2012) Learning to discover social circles in ego networks.
\newblock \emph{Adv. NIPS} pp 548--556.

\bibitem{hoff2003random}
Hoff PD
\newblock (2003) Random effects models for network data.
\newblock In \emph{Dynamic Social Network Modeling and Analysis} (Natl. Acad. Press) pp 303--312.

\bibitem{Zanghi:2010:CBR:1774858.1774960}
Zanghi H, Volant S, Ambroise C
\newblock (2010) Clustering based on random graph model embedding vertex features.
\newblock \emph{Pattern Recogn. Lett.} 31:830--836.

\bibitem{RSSA:RSSA471}
Handcock MS, Raftery AE, Tantrum JM
\newblock (2007) Model-based clustering for social networks.
\newblock \emph{J. Royal Stat. Soc. Ser. A} 170:301--354.

\bibitem{Yang:2009:CLC:1557019.1557120}
Yang T, Jin R, Chi Y, Zhu S
\newblock (2009) Combining link and content for community detection: a discriminative approach.
\newblock \emph{Proc. KDD} pp 927--936.

\bibitem{Jure_multiple_attr}
Kim JL
\newblock (2012) Latent multi-group membership graph model.
\newblock \emph{arXiv preprint arXiv:1205.4546.}.

\bibitem{yang2013community}
Yang J, McAuley J, Leskovec J
\newblock (2013) Community detection in networks with node attributes.
\newblock \emph{ICDM} pp 1151--1156.

\bibitem{xu2012model}
Xu Z, Ke Y, Wang Y, Cheng H, Cheng J
\newblock (2012) A model-based approach to attributed graph clustering.
\newblock \emph{SIGMOD} pp 505--516.

\bibitem{Amini.et.al2013}
Amini A, Chen A, Bickel PJ, Levina EM
\newblock (2013) Pseudo-likelihood methods for community detection in large sparse networks.
\newblock \emph{Ann. Statist.}, 41:4 pp 2097--2122

\bibitem{hoang2014joint}
Hoang TA, Lim EP
\newblock (2014) On Joint Modeling of Topical Communities and Personal Interest in Microblogs.
\newblock \emph{Social Informatics} (Springer), pp 1--16.

\bibitem{viennet2012community}
Viennet E, {et~al.}
\newblock (2012) Community detection based on structural and attribute similarities.
\newblock \emph{Proc. VLDB Endow.} pp 7--12.

\bibitem{binkiewicz2014covariate}
Binkiewicz N, Vogelstein JT, Rohe K
\newblock (2014) Covariate assisted spectral clustering.
\newblock \emph{arXiv preprint arXiv:1411.2158}.

\bibitem{Silva:2012:MAC:2140436.2140443}
Silva A, Meira, Jr. W, Zaki MJ
\newblock (2012) Mining attribute-structure correlated patterns in large attributed graphs.
\newblock \emph{Proc. VLDB Endow.} 5:466--477.

\bibitem{Cheng:2011:CLA:1921632.1921638}
Cheng H, Zhou Y, Yu JX
\newblock (2011) Clustering large attributed graphs: A balance between structural and attribute similarities.
\newblock \emph{ACM Trans. Knowl. Discov. Data} 5:12:1--12:33.

\bibitem{smith2014partitioning}
Smith LM, Zhu L, Lerman K, Percus AG
\newblock (2014) Partitioning networks with node attributes by compressing information flow.
\newblock \emph{arXiv preprint arXiv:1405.4332}.

\bibitem{Bickel&Chen2009}
Bickel PJ, Chen A
\newblock (2009) A nonparametric view of network models and Newman-Girvan and other modularities.
\newblock \emph{Proc. Natl. Acad. Sci. USA} 106:21068--21073.

\bibitem{Glover:tabu}
Glover F
\newblock (1986) Future paths for integer programming and links to artificial intelligence.
\newblock \emph{Comput. Oper. Res.} 13:533--549.

\bibitem{de2011Pajek}
De~Nooy W, Mrvar A, Batagelj V
\newblock (2011) Exploratory social network analysis with Pajek.
\newblock (Cambridge Univ. Press)

\bibitem{lazega2001collegial}
Lazega E
\newblock (2001) The collegial phenomenon: the social mechanisms of co-operation among peers in a corporate law partnership.
\newblock (Oxford Univ. Press)

\end{thebibliography}

\end{document}